\documentclass[lettersize,journal]{IEEEtran}
\usepackage{amsmath,amsfonts}
\usepackage{algorithmic}
\usepackage{algorithm}
\usepackage{array}
\usepackage{booktabs}
\usepackage{subfigure}
\usepackage{textcomp}
\usepackage{stfloats}
\usepackage{url}
\usepackage{verbatim}
\usepackage{graphicx}
\usepackage{cite}
\usepackage{hyperref}

\usepackage{wrapfig}

\usepackage{amsmath}
\usepackage{amssymb}
\usepackage{mathtools}
\usepackage{amsthm}
\usepackage{color,soul}
\DeclareMathOperator*{\argmax}{argmax}

\theoremstyle{plain}
\newtheorem{theorem}{Theorem}

\theoremstyle{definition}

\theoremstyle{remark}

\begin{document}

\title{Robust and Efficient Communication for Multi-Agent Learning}

\author{Rafael Pina, Varuna De Silva and Corentin Artaud\\
Institute for Digital Technologies\\
Loughborough University London, United Kingdom\\
{r.m.pina, v.d.de-silva, c.artaud2}@lboro.ac.uk
}

\markboth{Journal of \LaTeX\ Class Files,~Vol.~14, No.~8, August~2021}%
{Shell \MakeLowercase{\textit{et al.}}: A Sample Article Using IEEEtran.cls for IEEE Journals}


\maketitle

\begin{abstract}
    Effective communication is a cornerstone of distributed intelligence in Multi-Agent Reinforcement Learning (MARL), yet ensuring that generated messages are both informative and robust to physical constraints remains a significant challenge. This paper introduces \textbf{M}ulti-\textbf{A}gent \textbf{R}egularized \textbf{C}ommunication (\textbf{MARC}), a novel framework inspired by information-theoretic principles of conditional mutual information. MARC employs an attention-based architecture coupled with a unique message regularization mechanism designed to minimize uncertainty regarding future system states, thereby inducing the learning of highly representative communication protocols. Crucially, we evaluate MARC under stringent communication bottlenecks and lossy channels, simulating the real-world constraints of autonomous robotic networks and decentralized systems. Our results demonstrate that MARC significantly outperforms state-of-the-art methods in complex cooperative domains. Furthermore, we provide a deep analysis of message characteristics, proving that MARC maintains high operational performance even under significant data compression, offering a scalable path for deploying intelligent agents in resource-constrained environments. 
\end{abstract}


\section{Introduction}\label{sec:intro}
In distributed intelligent systems, agents often operate under conditions of partial observability, where critical environmental information is partitioned across the network and cannot be directly sensed by a single entity \cite{liu_multi-agent_2021,foerster_learning_2016,luis_deep_marl_water_wiley_2024}. In such complex cooperative tasks, effective communication acts as a fundamental bridge, enabling agents to transcend local limitations by sharing encoded perceptions with their teammates to achieve a collective goal \cite{kim_communication_2021}.

Recent advancements in Multi-Agent Reinforcement Learning (MARL) have positioned emergent communication as a pivotal research frontier \cite{das_tarmac_2019,sukhbaatar_learning_2016,liu_multi-agent_2021,foerster_learning_2016}. Most existing frameworks utilize the Centralized Training with Decentralized Execution (CTDE) paradigm \cite{oliehoek_optimal_2008,kraemer_multi-agent_2016}. While CTDE provides a powerful training foundation, relying on a centralized oracle is often unfeasible in practical field deployments due to latency, privacy, or infrastructure constraints \cite{canese_multi-agent_2021,cheng_knowledge_transfer_wiley_2022}. Communication-enabled decentralized execution offers a more robust alternative; by broadcasting learned message encodings rather than raw observation data, agents can maintain high-level coordination without the need for a global controller \cite{du_corr_communication_2021,wang2023ac2c}. When communication is available, under the CTDE paradigm both execution and training can be improved since agents can broadcast what they see or know at a certain moment to the rest of the teammates (as portrayed in Fig. \ref{fig:app_envs}). However, a significant challenge remains: ensuring these learned messages are sufficiently representative to drive meaningful cooperation while remaining stable throughout the non-stationary learning process.

\begin{figure}[!t]
    \centering
    \subfigure[Lumberjacks]{\label{fig:app_envs_a}\includegraphics[width=0.40\columnwidth]{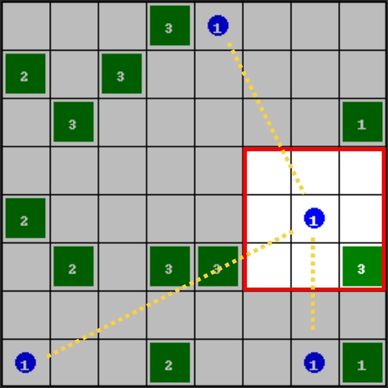}}
    \hspace{5mm}
    \subfigure[TrafficJunction]{\label{fig:app_envs_b}\includegraphics[width=0.40\columnwidth]{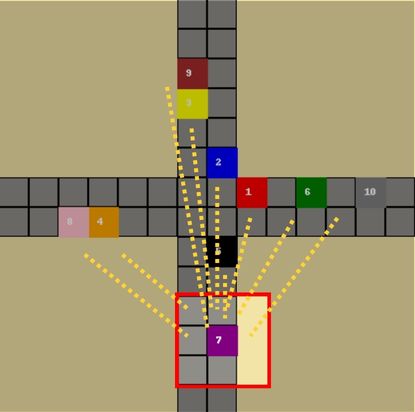}}
    \caption{Representation of two partially obserable environments where the agents can only see their surroundings and thus will strongly benefit if they receive information from the others about their local perceptions of parts of the environment that are far away (from \cite{magym}).}
    \label{fig:app_envs}
\end{figure}

Furthermore, the physical reality of communication channels---specifically limited bandwidth and signal noise---presents a major hurdle for autonomous systems \cite{resnick_capacity_2020,wang_bottleneck_2019,cheng_knowledge_transfer_wiley_2022}. Just as wireless communication protocols must compress data to fit within restricted spectral bands \cite{Mohamed2022}, intelligent agents must learn to generate information-dense messages that survive transmission through constrained or lossy channels.

In this paper, we propose \textbf{Multi-Agent Regularized Communication (MARC)}, a framework designed to produce highly representative communication protocols within the CTDE paradigm. MARC is modular and compatible with diverse value-function factorization methods. Our approach leverages an \textbf{attention-based architecture} enhanced by a novel \textbf{information-inspired regularizer}. This mechanism stabilizes the learning process by coaxing agents to learn messages that minimize the uncertainty regarding their future underlying observations, effectively maximizing the utility of the shared information.

Addressing the necessity of bandwidth efficiency, we also explore a novel perspective on message compression via the Discrete Cosine Transform (DCT). While frequently utilized in classical signal processing, the DCT has been largely overlooked in recent MARL literature. We demonstrate that information-based lossy compression can significantly reduce message size without compromising system-level performance, offering a scalable solution for resource-constrained hardware. Other related communication methods often fail to deal with compression and learn efficient communication strategies \cite{wang_bottleneck_2019}. 

In summary, the primary contributions of this work are as follows:
\begin{itemize}
    \item \textbf{A Novel Regularized Architecture:} We introduce MARC, featuring a unique message regularizer that enhances agent performance across diverse environments and ensures the generation of highly representative communication protocols.
    \item \textbf{Robustness under Compression:} We analyze the effects of lossy message compression using the DCT, demonstrating its efficacy in reducing communication overhead in complex MARL tasks.
    \item \textbf{Behavioral \& Message Analysis:} We provide an in-depth empirical study of learned message structures under both ideal and lossy conditions, showcasing the adaptive behaviors of the agents.
    \item \textbf{Theoretical Foundations:} We provide the information-theoretic grounding and motivations that underpin the proposed regularization mechanism.
\end{itemize}

The rest of this paper is organised as follows: in section \ref{app:rel_work} we introduce other relevant works that relate to this paper, and in section \ref{sec:bkgd} we present the key principles to provide the needed context to the concepts introduced in this work. In section \ref{sec:meth} we present in detail the proposed methods in this paper, followed by the experiments in section \ref{sec:exps}. Finally, in section \ref{sec:conc} we present the conclusions of this paper, together with some potential avenues for further research.

\section{Related Work}
\label{app:rel_work}
MARL has been well studied in the recent past \cite{liu_multi-agent_2021,wang2021qplex,sunehag2017valuedecomposition,hu_momix_2023,liu_pet_disent_2024}. Particularly in cooperative MARL, value function factorisation methods represent a branch of popular algorithms whose core objective is to learn a decomposition of a joint Q-function into agent-wise Q-functions. Numerous approaches that follow this configuration have shown outstanding results in multiple complex scenarios \cite{rashid_qmix_2018,shen2022resq,liu_pet_disent_2024,zhou_factorized_2019}. Importantly, these methods operate under the well-known CTDE paradigm. While this configuration represents the middle of the learning spectrum, the other two ends have also been studied, but it has been discussed that neither full centralisation nor full decentralisation can achieve as good results \cite{sunehag2017valuedecomposition}.  

As an alternative to the conventional methods, communication in MARL has gained the attention of the scientific community \cite{liu_multi-agent_2021,jiang_learning_2018,foerster_learning_2016,das_tarmac_2019,zhang_efficient_comm_2019,zhang_succint_comm_2020}. At its core, the key difference in these methods when compared to conventional approaches is that agents can share something about their experience of the environment, both during training and execution. In \cite{foerster_learning_2016} the authors have introduced one of the very first communication-based MARL methods. In simple terms, it is shown how two agents can solve tasks where they must communicate some of what one knows but the other doesn’t. In \cite{sukhbaatar_learning_2016}, the authors show the effect of communication in more complex scenarios with more agents. In particular, it is shown that, when agents communicate, their level of efficient cooperation increases in scenarios with more complex objectives, such as traffic junction negotiations. As these approaches became more popular, more complex methods started to arise. These consider factors such as what, whom or when to communicate. For example, in \cite{jiang_learning_2018} it is proposed an attention-based mechanism that allows the agent to decide to whom they should send their messages. In the case of \cite{das_tarmac_2019}, the authors propose an improved messaging mechanism that consists of signing the messages, allowing to target specific agents with customised information. Also in \cite{chu_multi-agent_2020}, the authors propose a communication method that uses message fingerprints and a different way of aggregating all the messages of the agents. In \cite{masia_2022} it is also shown a different way of aggregating messages in a manner that represents them with information that is more relevant for the policies. That is achieved through an aggregation of all the observations with the aim of predicting a global state representation. What is shared can also be important, as demonstrated in works such as \cite{liu_multi-agent_2021} or \cite{kim_communication_2021}, where agents that share their intentions achieve better performances. 

Still involving MARL but on a slightly different scope, language discovery in communication games has also been aim of studies. Usually in games that involve a speaker and a listener agent, the kind of language that emerges during the communication process has been studied in works like \cite{mordatch_emergence_2018}, where the agents learn a task-specific vocabulary during training that they use to solve the tested games. Also in \cite{gupta_networked_2020} or \cite{ivana_emerg_comm_2020} the authors evaluate the interpretability of the languages learned by the agents in terms of how humans can interpret them. Another step towards interpretability is to ground language, as it is studied in \cite{lin_ae_2021} where the authors use an autoencoder module that grounds language by reconstructing observations. While this deeper analysis becomes more challenging as more agents are used, it is still important to analyse what the agents devise to communicate in order to understand their learning process and how they understand the imposed tasks. 

Compressing messages in MARL has also been aim of studies. For instance, \cite{kim2019learning} proposes a communication architecture that can schedule messages in a way such that it can optimally operate under bandwidth contrainsts. This type of bandwidth limitations are relevant in MARL studies \cite{resnick_capacity_2020,wang_bottleneck_2019}, which might benefit from message compression in communication. In this work we tackle the problem from a different perspective and study how the DCT can help in compression for MARL.

\section{Background}
\label{sec:bkgd}
\subsection{Decentralised Partially Observable Markov Decision Processes (Dec-POMDPs)}
In this work, we model the learning problems following Decentralised Partially Observable Markov Decision Processes (Dec-POMDPs) \cite{oliohek_dec_pomdp_2016}. The Dec-POMDP can be defined as a tuple $G=\langle S,A,O,Z,P,r,\gamma,N\rangle$, where $s \in S$ represents the current state of the environment. From the current state, local observations $o_i$ can be derived according to a function $O(s,i):S\times \mathcal{N}\rightarrow Z$ for a certain agent $i \in \mathcal{N}\equiv \{1,\dots,N\}$. Additionally, the agents also maintain an action-observation history $\tau_i\in \mathcal{T}\equiv(Z\times A)^*\rightarrow\{\tau_1,\ldots,\tau_N\}$. When in a given state, each agent performs an action $a_i \in A$ where $A$ is the action space, forming a joint action $a=\{a_1,\ldots,a_N\}$ that is executed in the current state $s$ of the environment, and from which it results a reward that is shared by the entire team, $r(s,a):S\times A\rightarrow\mathbb{R}$. After the actions are executed, the environment transits to a next state $s'$ according to a probability function that models the dynamics of the environment, $P(s'|a,s):S\times A\times S\rightarrow [0,1]$. Usually, the actions of the agents are controlled by a policy $\pi_i(a_i|\tau_i):\mathcal{T}\times A \rightarrow [0,1]$, but in the considered context of communication in MARL, the policy is instead conditioned not only on $\tau_i$, but also on a set of incoming messages from the other agents $m_{-i}$, meaning that the corresponding policy that controls the actions of the agents can be written as $\pi_i(a_i|\tau_i,m_{-i})$ (where $-i$ refers to all except $i$). During learning, the joint objective of the agents is to maximise an action-value function $Q_{\pi}(s_t,a_t)=\mathbb{E}_{\pi}[R_t|s_t,a_t]$, where $R^t=\sum_{k=0}^\infty\gamma_kr_{t+k}$ is the discounted return with a discount factor $\gamma \in [0,1)$. 

\subsection{Centralised Training with Decentralised Execution (CTDE) and Communication in MARL}
Within MARL, centralised training with decentralised execution (CTDE) is a popular paradigm that allows the agents to be trained in a centralised manner, but they must execute their policies in a decentralised way \cite{li_structured_rl_2022,sunehag2017valuedecomposition,hu_momix_2023}. Communication-based methods can be easily integrated into this framework. The key difference from simple CTDE is that, with communication, the agents can broadcast information to others not only during training but also during execution. This means that, despite the learned policies being decentralised, the agents can receive encoded messages that represent what their teammates see or sense at a certain timestep. While in the fully centralised setting there is a common oracle that sees everything in the environment, a communication setting can be seen as a more realistic choice since communication is done on a peer-to-peer basis where each agent is responsible for broadcasting its own messages. Thus, they are not dependent on a central unit that has the big - and often unrealistic - advantage of observing everything at the same time.  

Intuitively, communication in MARL should improve the performances of the agents when it is combined with other base methods, as shown in works such as \cite{liu_multi-agent_2021}. If the messages learned are informative enough, they should be useful for the agents to learn to “talk” with each other and improve their cooperative strategies. However, if the messages learned are not adequate, this might result in an overload for the learning networks and harm the performances of the agents. 

\subsection{Value Function Factorisation in MARL}\label{sec:value_func_marl}
In cooperative MARL, value function factorisation methods form a group of powerful algorithms to solve complex MARL tasks. The key idea of these methods is to learn a way of decomposing a joint action-value function into agent-wise functions \cite{sunehag2017valuedecomposition},
\begin{equation}\label{eq:value_function_fact}
    Q_{tot}\left(\tau ,a\right)\mathrm{=}f\left(Q_i\mathrm{(}\tau_i,a_i\mathrm{;}{\theta }_i\mathrm{)}\right),\forall i \in\{1,\dots,N\},
\end{equation}
where $f$ here represents a certain function that mixes the individual functions into a joint function $Q_{tot}$. An efficient decomposition of the joint action-value function should be done in a way that satisfies the Individual-Global-Max (IGM) condition \cite{qtran_2019}. This condition states that the set of local optimal actions should also maximise the joint Q-function. This can be formalised as 
\begin{equation}\label{eq:igm}
    \argmax_aQ_{tot}\left(\tau,a\right)=
    \begin{pmatrix}
    \argmax_{a_1}Q_1({\tau }_1,a_1) \\
    \vdots\\
    \argmax_{a_N}Q_N({\tau }_N,a_N)
    \end{pmatrix}.
\end{equation}
In \cite{sunehag2017valuedecomposition}, the authors introduce Value Decomposition Networks (VDN) as a way of factorising the joint $Q_{tot}$ as the sum of the individual Q-functions. Later on, QMIX \cite{rashid_qmix_2018} proposes a new non-linear way of factorising the $Q_{tot}$ that extends the range of functions that can be represented by VDN to a family of monotonic functions. Both these factorisation methods are sufficient to satisfy (\ref{eq:igm}).

In these methods, the loss used to update the networks is based on the temporal difference loss as described in DQN \cite{Mnih2015HumanlevelCT} that uses a replay buffer and a target network, but here with respect to a $Q_{tot}$. This loss can be formalised as 
\begin{equation}\label{eq:vff_loss}
    \mathcal{L}(\theta)=\mathbb{E}_{b\sim B}\left[\big(r+\gamma\mathop{\mathrm{max}}_{a'}Q_{tot}(\tau',a';\theta^-)-Q_{tot}(\tau,a;\theta)\big)^2\right],
\end{equation}
for some sample $b$ that is sampled from the replay buffer $B$, and where $\theta$ and $\theta^-$ are the parameters of the learning network and of a target network, respectively. In this paper, we use VDN and QMIX to demonstrate our communication approach on top of standard MARL approaches that don't initially use communication.

\subsection{Discrete Cosine Transform (DCT)}
In the fields of data compression and signal processing, the Discrete Cosine Transform (DCT) \cite{dct_1974} is a certain function based on a sum of cosine functions that processes a sequence of data points and encodes them into a compressed representation. By doing so, it is possible to achieve a much smaller representation of the data in terms of size occupied. For simplicity, we introduce here only the expression for the type II of the DCT (that is the most common form and the one used in this paper) that, generally, can be done following \cite{makhoul_dct_1980}
\begin{equation}\label{eq:c5_dct}
    X(k)=2\sum_{n=0}^{C-1}x(n)\text{cos}\left(\frac{\pi(2n+1)k}{2C}\right), 0\leq k\leq C-1,
\end{equation}
where $X(k)$ represents the $k^{th}$ transform of the DCT, and $x(n)$ denotes the sequence of values to be compressed, with size $C$. In the context of this paper, the DCT is applied to the last dimension of the vectors containing the messages of the agents, forming a compressed representation of their messages. Using this smaller representation is important when we consider channel sizes or other constraints, through which only compressed representations can be sent. When this representation arrives at the other end of the communication channel, the inverse of the function is applied, and the message can be recovered with some information loss. We hypothesise that, if this information loss is not too heavy, the agents can still learn and benefit from communication in MARL. This decompressing process can be done by inverting Eq. (\ref{eq:c5_dct}). The inverse can be calculated following \cite{makhoul_dct_1980}
\begin{align}\label{eq:c5_idct}
    x(n)=\frac{1}{C}\Biggl[\frac{X(0)}{2}\Biggl. &\left.+\sum_{k=1}^{C-1}X(k)cos\left(\frac{\pi(2n+1)k}{2C}\right)\right], \notag \\ & 0\leq n \leq C-1.
\end{align}

The DCT can be used both as a lossy and as a lossless compression method, depending on the parts of the messages that are encoded. In this paper, we consider the case of lossy compression, since for lossless compression to be achieved the messages are encoded but without reducing their sizes, which becomes redundant when we consider limited communication channels.

\section{Methods}
\label{sec:meth}
\subsection{Multi-Agent Regularized Communication (MARC) in MARL}
In this section, we propose MARC, a new architecture for efficient communication in MARL. The idea is to build an inter-agent communication architecture that can learn meaningful messages to ensure cooperation in complex MARL tasks. In the proposed architecture, MARC starts by using an attention module to learn messages that are generated from the local observations of the agents. This mechanism starts by encoding the observations that are then given as an initial message $m'$ to an attention module, in order to learn their relative importance. These values are embedded into $k\in\mathbb{R}^{d_k},v\in\mathbb{R}^{d_k},q\in\mathbb{R}^{d_k}$, where $d_k$ is the embedding dimension (as shown in Fig.\ref{fig:net_arch}). As such, we define the keys, queries, and values, at each timestep $t$, for the attention operations as
\begin{equation}
    k_t = \left[W_{K,1}\tau_1^t,\dots,W_{K,i}\tau_i^t,\dots,W_{K,N}\tau_N^t\right],
\end{equation}
\begin{equation}
    v_t = \left[W_{V,1}\tau_1^t,\dots,W_{V,i}\tau_i^t,\dots,W_{V,N}\tau_N^t\right],
\end{equation}
\begin{equation}
    q_t = \left[W_{Q,1}{m_1'^t},\dots,W_{Q,i}{m_i'^t},\dots,W_{Q,N}{m_N'^t}\right],
\end{equation}
where $W_{Q,i}, W_{K,i}, W_{V,i}$ are trainable weight matrices, and where $m'^t$ corresponds to the initial message embeddings that will be refined (as in Fig. \ref{fig:net_arch}). As a result, the attention weights calculated for an agent $i$ from the messages and observations in our approach can be formalised as
\begin{equation}
    \alpha_{ij}=\frac{\text{exp}(\phi\cdot(q_i^t\cdot{k_j^t}^T))}{\sum_{x\in\mathcal{N}}\text{exp}(\phi\cdot(q_i^t\cdot{k_x^t}^T))},
\end{equation}
where $\phi$ is a scaling factor. These weights are further used to calculate a new aggregated attention-encoded message $m_i^t$,
\begin{equation}
    m_i^t=\sum_{j=1}^N\alpha_{ij}v_j^t,
\end{equation}
where the weights $\alpha_{ij}$ define the relationship between agents $i$ and $j$, and hence the importance of value $v_j^t$. The intuition behind the choice of keys, queries and values, is that, by relating the messages to multiple different latent representations of the observations, the messages learned will be able to capture more relevant information from the observations.
\begin{figure*}[!t]
    \centering
    \includegraphics[width=\textwidth]{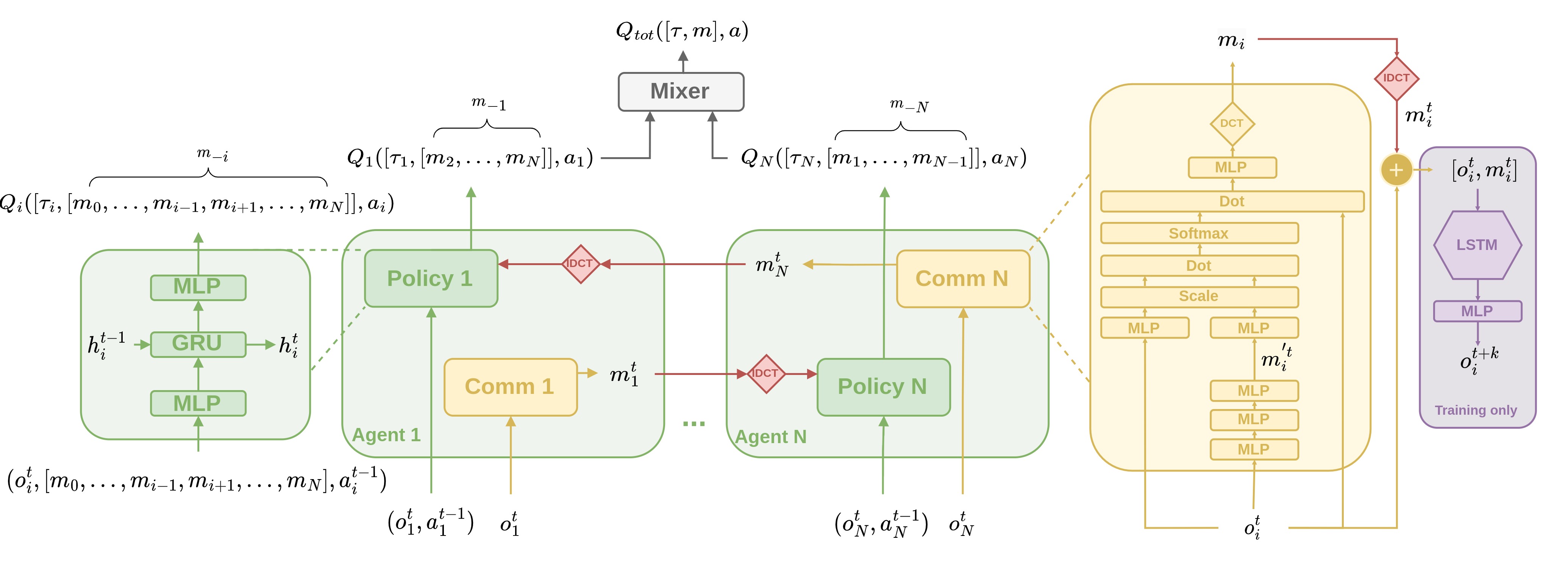}
    \caption{Architecture of the proposed communication method for MARL. The proposed method can be used together with any value function factorisation method (whose mixer is represented in the figure in the block \textit{mixer}) and uses parameter sharing. The pink diamonds represent the inverse DCT that is applied to decompress the messages that are compressed with the DCT (yellow diamond) after they are computed by the communication network (yellow). Note that both the DCT and IDCT blocks are only used when we analyse the effect of compression in section \ref{sec:exps_comp}. In the main experiments, these blocks are not applied.}
    \label{fig:net_arch}
\end{figure*}

Despite the recent success of attention-based architectures for multi-agent communication in complex scenarios, we begin with the hypothesis that the messages originated by our base architecture are not rich enough to learn complex environments. In this sense, we propose a message regularizer in our architecture. The key intuition is that, if the messages generated by the agents help to predict their own next observations, it means that these messages are likely to contain more meaningful information about their individual observations. We hypothesise that the uncertainty associated with the future values of the observations is reduced when we use both the previous values of the messages and the observations, when compared to using only the previous values of the observations. Motivated by the concepts of conditional entropy, our hypothesis is formalised in the following Theorem \ref{theo:2}, when we consider a certain set of messages that contains information about the observations.

\begin{theorem}
    \label{theo:2}
    Let $m$ be a certain encoding of the observations $o$ and $m^-$ and $o^-$ represent the previous values of each, respectively. We have that $H(o|o^-)\geq H(o|o^-,m^-)$.
\end{theorem}
\begin{proof}
    Since $m$ is a certain encoding of $o$, it is true that it contains information about $o$. The mutual information between the observations and the previous messages when conditioned on the previous observations can be written as
    \begin{equation}
        I(o,m^-|o^-)=H(o,o^-)+H(m^-,o^-)-H(o,o^-,m^-)-H(o^-)
    \end{equation}
    \begin{equation}
        =H(o|o^-)-H(o|o^-,m^-)
    \end{equation}
    Since the mutual information is always non-negative, from the above we can write:
    \begin{align}
        &H(o|o^-)-H(o|o^-,m^-)\geq 0\\
        &H(o|o^-)\geq H(o|o^-,m^-),
    \end{align}
    meaning that conditioning $o$ both on $o^-$ and $m^-$ reduces the uncertainty linked to the observations when compared to conditioning only on $o^-$.
\end{proof}

Motivated by Theorem \ref{theo:2}, we will define a message regularizer that aims to predict the next observations of the agents, given the previous messages received together with the previous observations. As such, we use a recurrent encoder that receives the previous messages alongside the previous observations, $[o_i^{-p}, m_i^{-p}]$ with values up to $T-p$, and predicts the next $p$ observations ahead $o_i^{+p}$, where $p$ denotes the number of timesteps to predict ahead of each timestep $t$, and $T$ the total length of the episode. For our experiments, we predict only one timestep ahead of each $t$. Let $g(\cdot;\theta_r)$ here denote a certain neural network with parameters $\theta_r$ composed of an LSTM module that estimates the next values of the observations of the agents given the previous messages and the previous observations (as described in Fig. \ref{fig:net_arch}). The observations predicted by this network for the timesteps ahead can be given by $o_i^{+p'}=g([o_i^{-p}, m_i^{-p}];\theta_r)$. The predicted outputs of this network can then be used to calculate a second loss that will auxiliate our learning problem, as described in
\begin{equation}
    \mathcal{L}_m = \frac{1}{T-p}\sum_{t=1}^{T-p}\lVert o^{+p} - o_i^{+p'}\rVert_2^2,
\end{equation}
for an agent $i$, where $T-p$ denotes the number of predicted values in the vector, given that an episode lasts $T$ timesteps and $g(\cdot;\theta_r)$ is predicting $k$ timesteps ahead of each timestep $t$. For the proposed method, this additional loss is used alongside the loss described in Eq. (\ref{eq:vff_loss}), resulting in the overall objective of minimising, with respect to $\tau$, the following loss
\begin{align}
    \mathcal{L}(\theta,\theta_c,\theta_r)=&\sum_{b=1}^B\Biggl[\big(y_{tot}-Q_{tot}(\tau,a;\theta)\big)^2\Biggl. \notag
    \\
    &\left.+\frac{1}{T-p}\sum_{t=1}^{T-p}\lVert \tau^{+p}-g(\tau^{+p},m^{+p};\theta_r)\rVert_2^2\right],
\end{align}
for a batch of samples $B$, and where $y_{tot}=r+\gamma\mathop{\mathrm{max}}_{a'}Q_{tot}(\tau',a';\theta^-)$, for the parameters of a network and the respective target, $\theta$ and $\theta^-$, the parameters of the communication network $\theta_c$, and the parameters of the regularizer module $\theta_r$. Ahead we demonstrate how this message regularizer positively affects the learning process of the agents. Fig. \ref{fig:net_arch} depicts the architecture of the proposed method as described in this section (note that the DCT and IDCT blocks are represented for completeness and are not used when compression is not being analysed). Importantly, we build our communication method on top of value function factorisation methods and adopt the famous parameter sharing convention \cite{gupta_2017}. As such, our method can be easily integrated with any existing value function factorisation method, to which we add our regularizer objective without affecting the convergence of the algorithm, as shown in Theorem \ref{theo:1} below. For simplicity of notation, we do not consider the history of the agents in this demonstration. While this is not the case in the experiments, for the theorem we assume tabular conditions.
\begin{theorem}\label{theo:1}
    Assuming that $0\leq \alpha_k < 1$, $\sum_k\alpha_k=\infty$, $\sum_k\alpha_k^2<\infty$, updating the learning problem following the rule
    \begin{align}\label{eq:q_update}
        Q_{tot}^{k+1}&(o_t,a_t)=Q_{tot}^k(o_t,a_t)+\alpha_k\Biggl[r_t+\gamma max_a Q_{tot}^k(o_{t+1},a)\Biggr. \nonumber
        \\
        &\left.-Q_{tot}^k(o_t,a_t)+\frac{1}{T-p}\sum_{t=1}^{T-p}\lVert o^{+p}-g_k(o^{+p},m^{+p};\theta_r)\rVert_2^2\right].
    \end{align}
    will make $Q_{tot}$ converge to a certain value, such that $||Q_{tot}^k(o_t,a_t)-Q_{tot}^*(o_t,a_t)||\leq \lambda (T-p)G$, as $k\rightarrow \infty$ and $\lVert o^{+p}-g_k(o^{+p},m^{+p};\theta_r)\rVert_2^2\leq G$.
\end{theorem}
\begin{proof}
    Let $y_{tot}=r_t+\gamma max_a Q_{tot}^k(o_{t+1},a)$ for conciseness. Let $U_k(\cdot)$ denote the term $\sum_{t=1}^{T-p}\lVert o^{+p}-g_k(o^{+p},m^{+p};\theta_r)\rVert_2^2$. By rearranging the previous equation, we can write
    \begin{align}
        &Q_{tot}^{k+1}(o_t,a_t)\nonumber \\ 
        &=Q_{tot}^k(o_t,a_t)+\alpha_k\left[y_{tot}-Q_{tot}^k(o_t,a_t)+\frac{1}{T-p}U_k(\cdot)\right]\nonumber \\
        &=Q_{tot}^k(o_t,a_t)+\alpha_k\left[y_{tot}-Q_{tot}^k(o_t,a_t)-(p-T)^{-1}U_k(\cdot)\right] \nonumber \\
        &=Q_{tot}^k(o_t,a_t)+\alpha_k\left[y_{tot}-Q_{tot}^k(o_t,a_t)-\lambda U_k(\cdot)\right]
    \end{align}
    for $\lambda=(p-T)^{-1}$. Similarly to \cite{zhang_efficient_comm_2019} and based on \cite{jaakola_dynamic_processes_1993}, we can now write the equation as a function of stochastic dynamic processes,
    \begin{equation}
        \delta^{k+1}(o_t,a_t)=(1-\alpha_k)\delta^k(o_t,a_t)+\alpha_k\left[F_k(o_t,a_t)-\lambda U_k(\cdot)\right],
    \end{equation}
    where $\delta^k(o_t,a_t)=Q_{tot}^k(o_t,a_t)-Q^*_{tot}(o_t,a_t)$, and $F_k(o_t,a_t)=r_t+\gamma max_a Q_{tot}^k(o_{t+1},a)-Q_{tot}^k(o_t,a_t)$. This can be decomposed into random processes according to $\delta^{k+1}(o_t,a_t)=\delta_1^{k+1}(o_t,a_t)+\delta_2^{k+1}(o_t,a_t)$, resulting in
    \begin{align}
        &\delta_1^{k+1}(o_t,a_t)=(1-\alpha_k)\delta_1^{k}(o_t,a_t)+\alpha_k F_k(o_t,a_t)\\
        &\delta_2^{k+1}(o_t,a_t)=(1-\alpha_k)\delta_2^{k}(o_t,a_t)-\alpha_k U_k(o_t,a_t).
    \end{align}
    From \cite{melo_2001_convergence} it is known that the first process converges to zero w.p. 1. Similarly to \cite{zhang_efficient_comm_2019}, the second process can be written as
    \begin{equation}
        \delta_2^{k+1}(o_t,a_t)\leq(1-\alpha_k)||\delta_2^{k}(o_t,a_t)||+\alpha_k\lambda(T-p)G,
    \end{equation}
    which means that (\ref{eq:q_update}) converges to a number greater than zero. Thus, $Q_{tot}$ will converge to an optimal value as the iterations of the update tend to infinity, i.e., 
    \begin{align}
        ||Q_{tot}^k(o_t,a_t)-Q_{tot}^*(o_t,a_t)||&=||\delta^k(o_t,a_t)|| \nonumber \\
        &=||\delta_1^k(o_t,a_t) + \delta_2^k(o_t,a_t)||\nonumber \\
        &\leq||\delta_1^k(o_t,a_t)|| + ||\delta_2^k(o_t,a_t)||\nonumber \\
        &\leq \lambda (T-p)G.
    \end{align}
    Hence, this means that the proposed method will still lead to convergence of the learning process, as the number of iterations $k$ tends to infinity. 
\end{proof}

\subsection{Message Compression for Lossy Communication in MARL}
In this work, we also intend to study whether communication can still be done efficiently when the messages exchanged are, for example, lost, or dropped during the process. Hence, instead of simply zeroing or naively cutting values of the messages, we use the DCT, a popular lossy compression method. We note that this method for message compression does not require changing the sizes of any of the agent networks, since the compression only applies to a potential communication channel, and the messages are reconstructed to the original size when they reach the other communication end. 

In our architecture, each agent will generate its messages according to the method described in the previous subsection, and then, in the results presented in subsection \ref{sec:exps_comp}, the generated messages will be compressed using the DCT, as described in Eq. (\ref{eq:c5_dct}). In this work, we focus on lossy compression, i.e., each agent compresses the message in a way that will make it smaller in size, reducing the potential communication overhead. When the messages arrive at the destination, these are decompressed using the inverse of the DCT (IDCT, as described in Eq. (\ref{eq:c5_idct}). This is depicted in the architecture in Fig. \ref{fig:net_arch}). When we consider message compression, the change in the architecture comes from the compression of the message after it is generated (DCT in the yellow diamond of the figure), and then it is decompressed in the destination (IDCT, illustrated by the pink diamonds in the figure). Finally, in section \ref{sec:exps_comp} we use the result of this process as input of the agents, together with their own observations, in the same way that was described before, but now the messages are compressed in the source and then decompressed in the endpoint. This results in the input $[\tau_i, m_{-i}^*]$ for agent $i$, where $m_{-i}^*$ here represents the decompressed messages coming from the other agents.

\section{Experiments and Results\protect\footnote{Codes available at \href{https://github.com/rafaelmp2/marc-marl}{https://github.com/rafaelmp2/marc-marl}}}\label{sec:exps}
In this section, we present the experiments carried out to evaluate the performance of the proposed MARC. While the proposed communication method can be built on top of any value function factorisation method in MARL, in this paper we apply it together with the architectures of QMIX \cite{rashid_qmix_2018} and VDN \cite{sunehag2017valuedecomposition}, two popular value function factorisation methods in MARL, as introduced in subsection \ref{sec:value_func_marl}. We demonstrate the performances of MARC in comparison with the vanilla methods without communication, and three popular communication methods, MASIA \cite{masia_2022}, COMMNET \cite{sukhbaatar_learning_2016}, and TARMAC \cite{das_tarmac_2019}. Note that, in the results in \ref{sec:main_exps}, there is no compression, i.e., the DCT and IDCT blocks in Fig. \ref{fig:net_arch} are not used. These are only used when investigating message compression in \ref{sec:exps_comp}. We start by presenting the experimental details and hyperparameters below.

\subsection{Hyperparameters and Implementation Details}\label{app:hyperp}
Our experiments were executed in a desktop computer with an 18-core Intel Core i9-10980XE CPU, 256GB of RAM, and 3 NVIDIA A6000 GPUs. In the experiments carried out in this paper, the agents are controlled by a deep recurrent Q-network that uses a GRU (gated recurrent unit) with width 64. Additionally, in QMIX the hidden layers used by the mixer have size 32. All the experimented methods share the learning networks, following the parameter sharing convention widely adopted in the literature. This allows to speed up learning but requires that an agent ID is added to the inputs of the agents so that the network can differentiate them. Regarding the communication networks in MARC, we set the size of all the embeddings of the attention mechanism to 64. The length of the messages that each agent produces at each timestep is defined as 10. When using compression, a value of, for example, 40\% of compression will naturally result in messages of size 6, and so on. Regarding TARMAC and MASIA, we use them with QMIX as the mixer.

We use a replay buffer to store experiences with size 5000, from which minibatches with size of 32 episodes are sampled for training. The replay buffer is updated over time. To take actions, the exploration-exploitation trade-off of the agents follows the epsilon-greedy method, with the value epsilon, $\epsilon$, starting at 1. This value anneals gradually throughout 50000 training episodes down to a minimum of 0.05. The value of the discount factor $\gamma$ is set to 0.99. The parameters of the target networks are updated every 200 episodes. All the networks are trained using the RMSProp optimization algorithm, with a learning rate $\alpha=5\times10^{-4}$.

\begin{figure*}[!t]
    \centering
    \includegraphics[width=\textwidth]{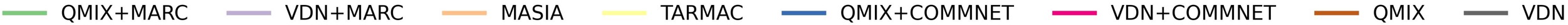}
    \\
    \vspace{0.00mm} 
    \subfigure[3s\_vs\_5z]{\label{fig:env_a}\includegraphics[width=0.23\textwidth]{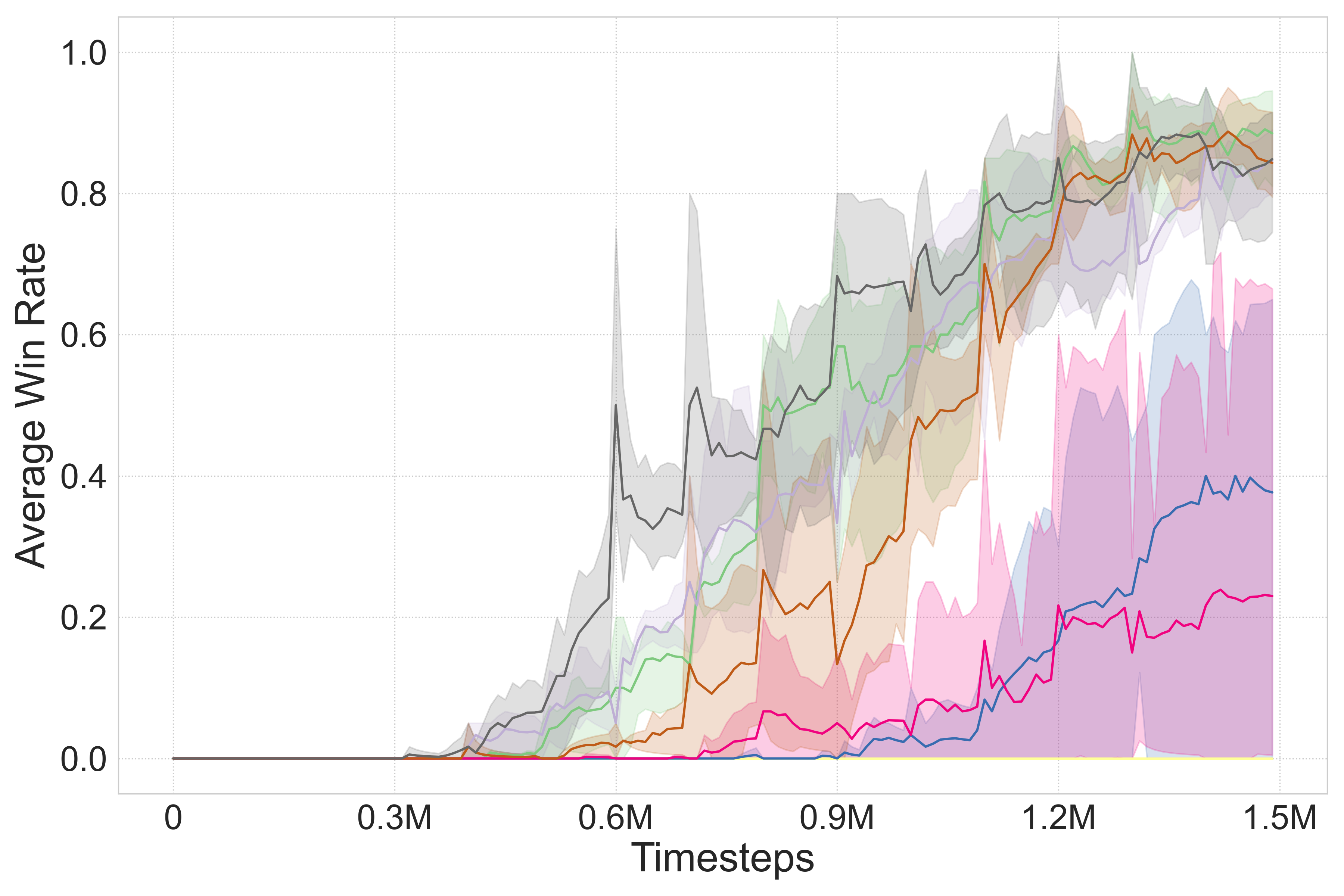}}
    \hfill
    \subfigure[2c\_vs\_64zg]{\label{fig:env_b}\includegraphics[width=0.23\textwidth]{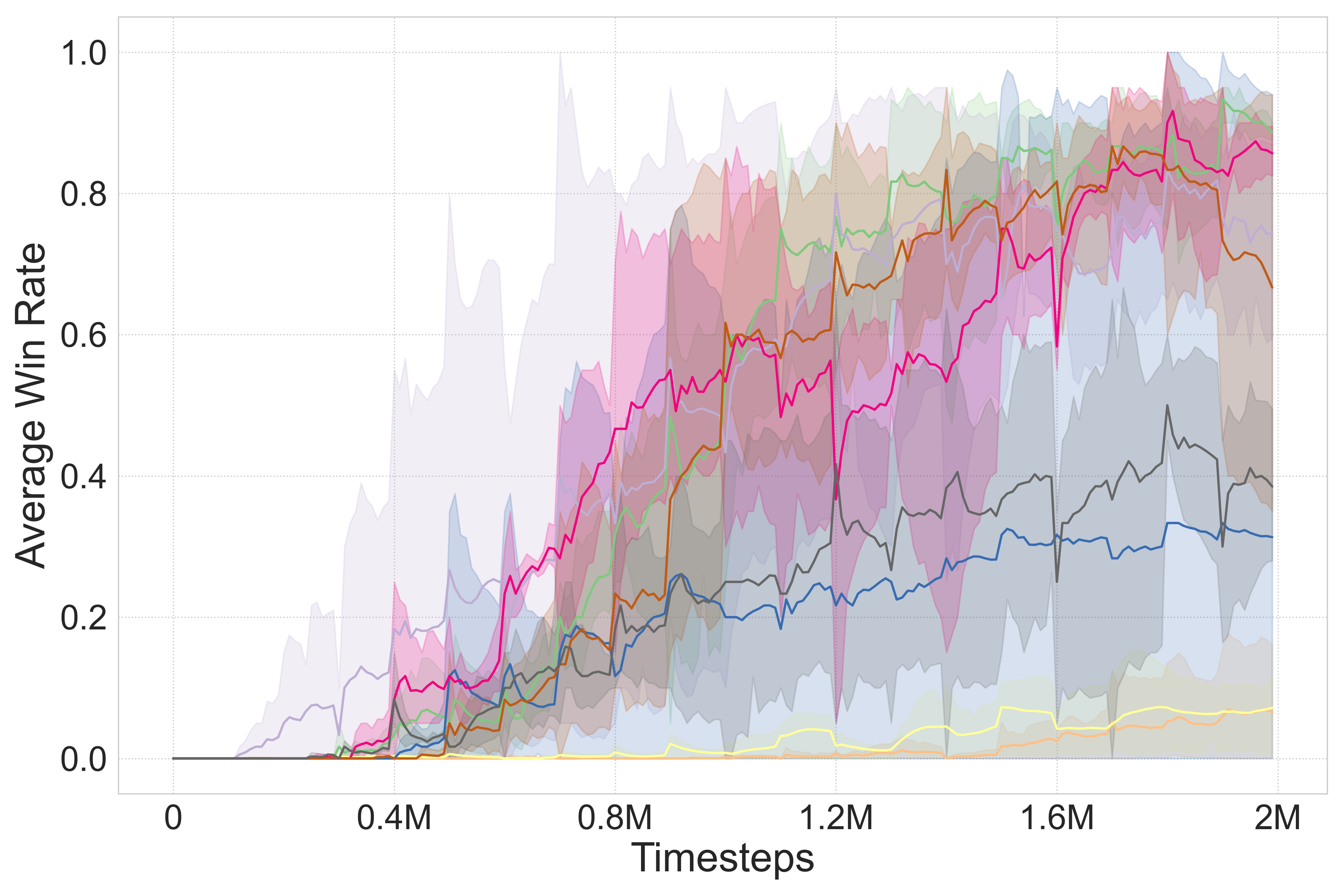}}
    \hfill
    \subfigure[MMM2]{\label{fig:env_c}\includegraphics[width=0.23\textwidth]{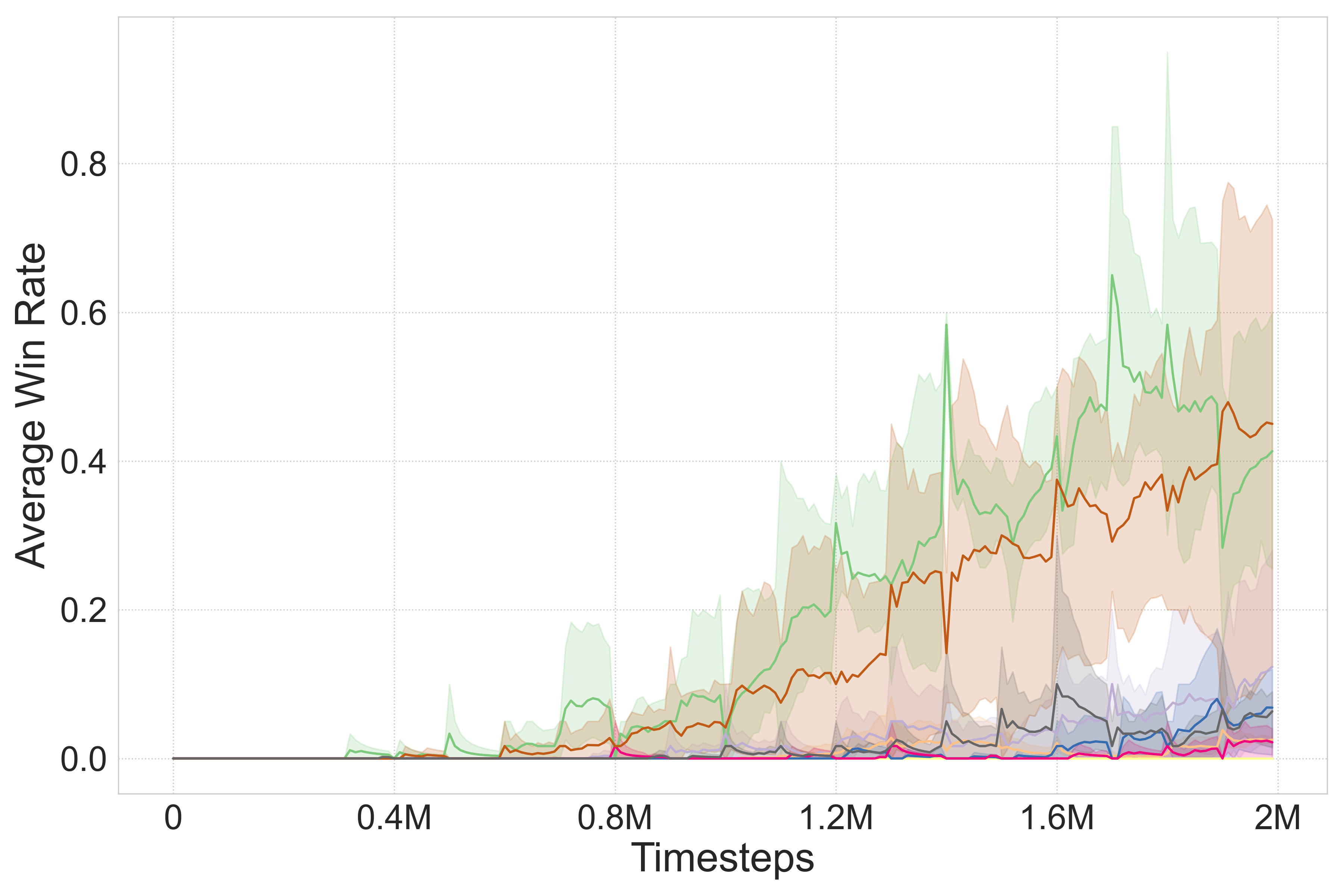}}
    \hfill
    \subfigure[1o2r\_vs\_4r]{\label{fig:env_d}\includegraphics[width=0.23\textwidth]{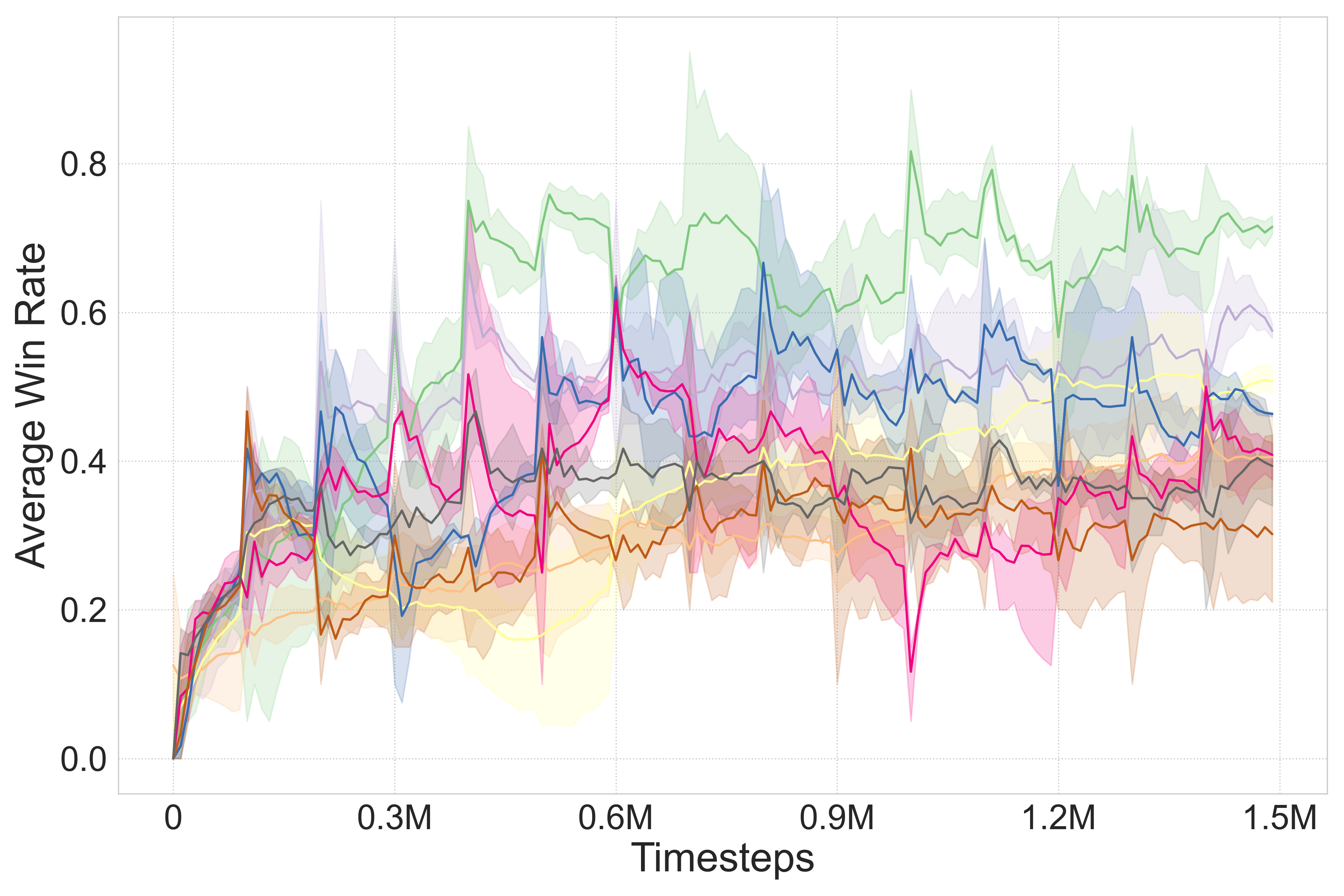}}
    \caption{Performance of the attempted methods in the Starcraft environments. The plots depict the average win rates of the agents over training time. (a) In 3s\_vs\_5z 3 stalker units must defeat a team with 5 zealot units. (b) In 2c\_vs\_64zg, the agents are 2 colossi playing against 64 zerglings. (c) MMM2 where 1 medivac (healing unit), 2 marauders and 7 marines play against 1 medivac, 3 marauders and 8 marines. (d) In 1o2r\_vs\_4r, 1 overseer and 2 roaches must defeat 4 reapers, where all units spawn at random points and only the overseer knows the enemy location (as in \cite{masia_2022}).}
    \label{fig:main_res}
\end{figure*}
\begin{figure*}[!t]
    \centering
    \includegraphics[width=\textwidth]{resources/legend.jpg}
    \\
    \vspace{0.00mm} 
    \subfigure[$p=-0.50$]{\label{fig:env_pp_a}\includegraphics[width=0.32\textwidth]{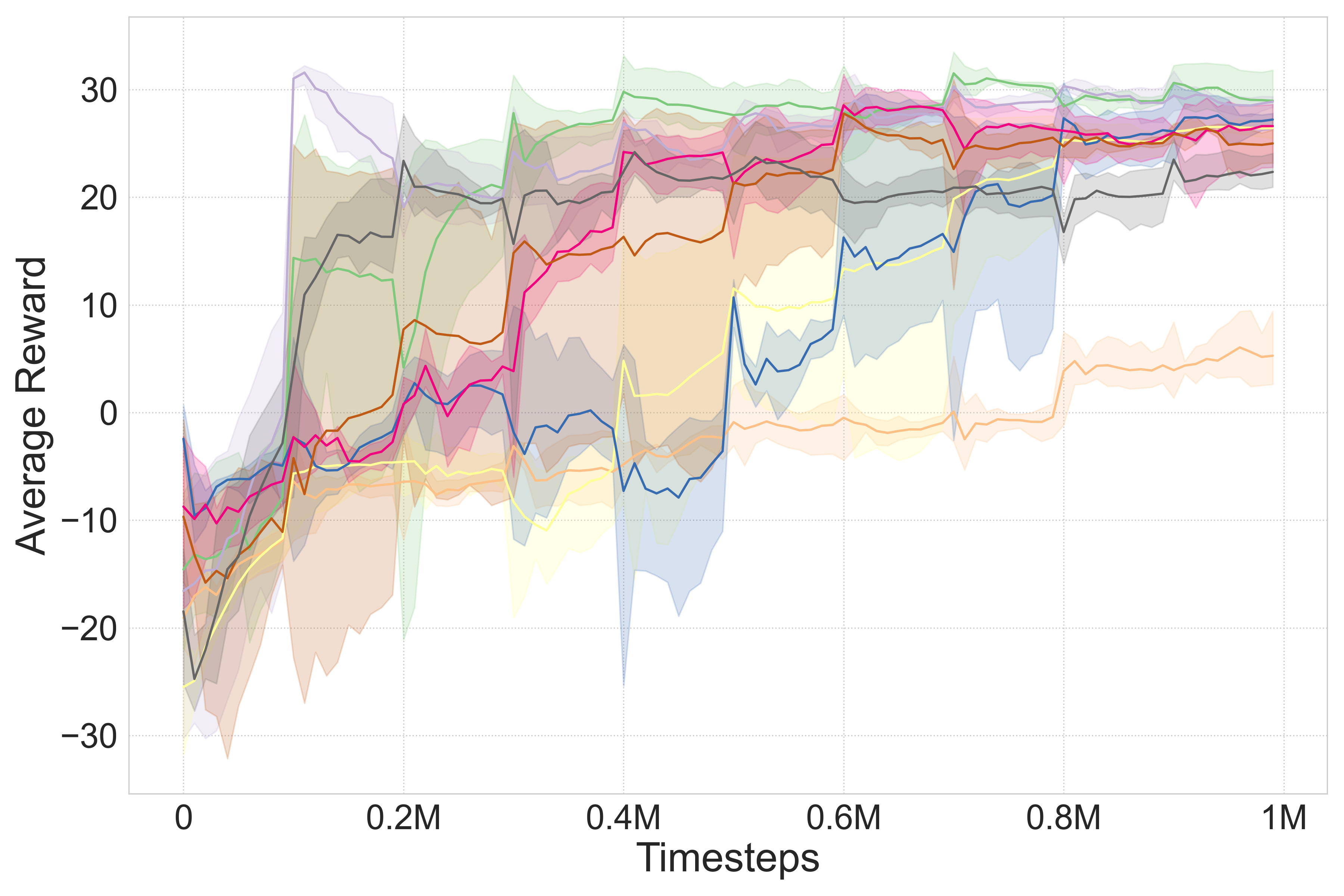}}
    \hfill
    \subfigure[$p=-0.75$]{\label{fig:env_pp_b}\includegraphics[width=0.32\textwidth]{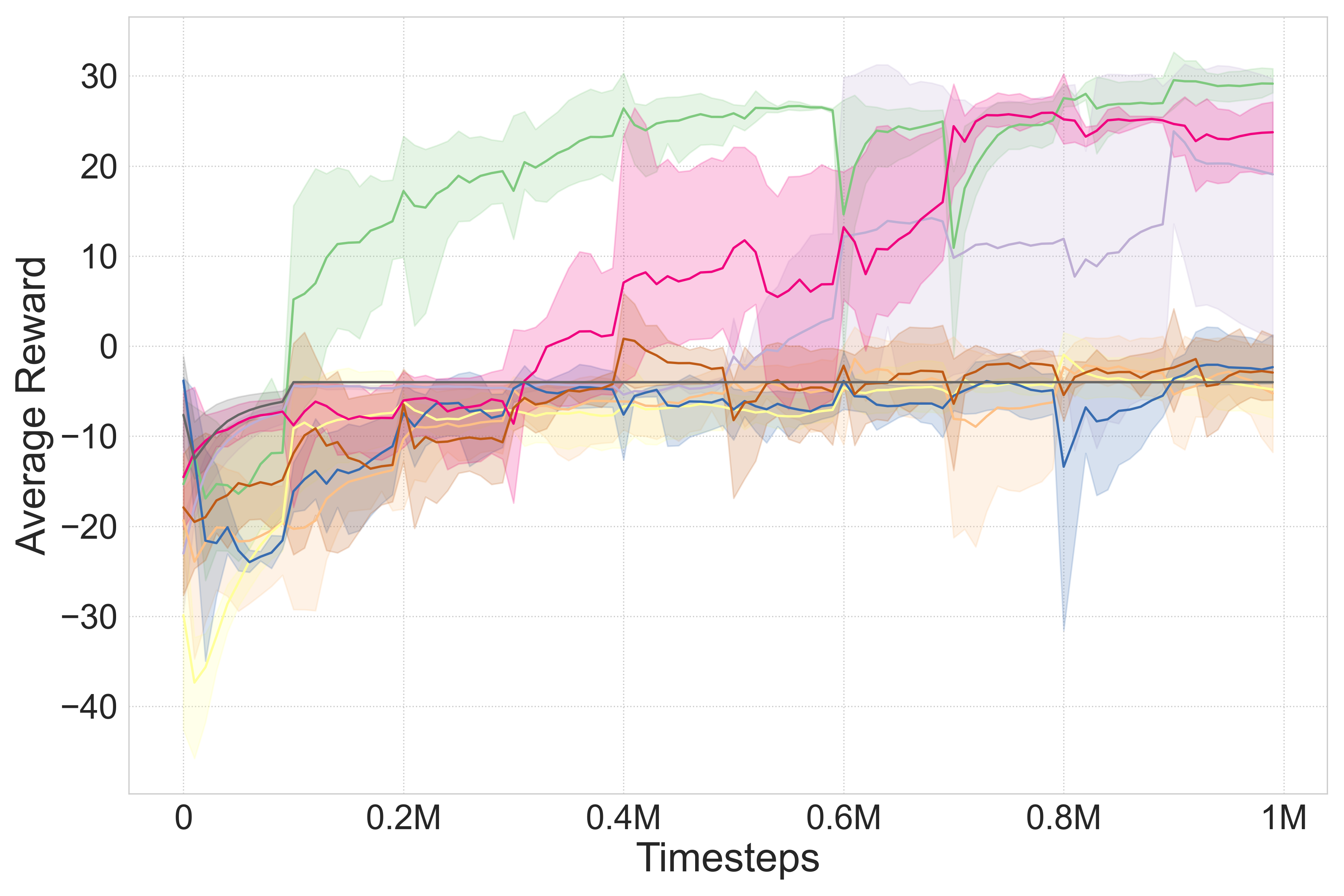}}
    \hfill
    \subfigure[$p=-1.0$]{\label{fig:env_pp_c}\includegraphics[width=0.32\textwidth]{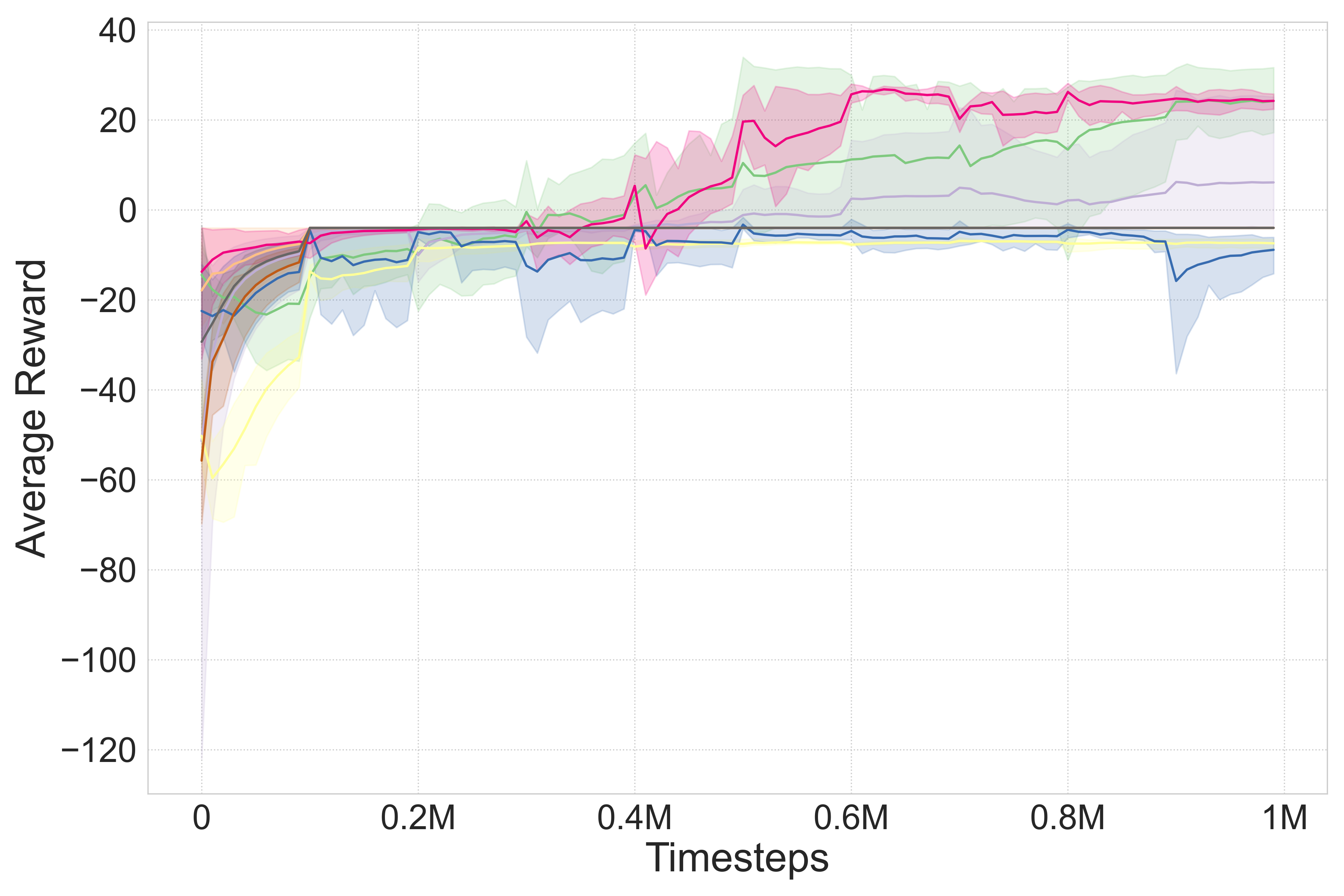}}
    \caption{Performance of the attempted methods in the PredatorPrey environment for different levels of cooperation penalty. The plots depict the average rewards of the agents over training time.}
    \label{fig:pp_res}
\end{figure*}
\subsection{Main Results: Performance in SMAC and PredatorPrey}\label{sec:main_exps}
Fig. \ref{fig:env_a} illustrates the performances in the 3s\_vs\_5z SMAC scenario \cite{smac_2019}. Considering that this is the simplest scenario out of all the experimented ones, it was expected that all the methods would be able to solve it somewhat easily. Particularly in the case of QMIX, we can see that the agents benefit from communication, making them solve the task sooner. However, in the case of VDN, it doesn’t show to be as useful. This suggests that, for this scenario, communication might not be as important as in other cases to learn efficient strategies, although convergence is still achieved quickly. 

Yet, in the other more complex scenarios communication proves to have a stronger impact. In Fig. \ref{fig:env_b}, we can see that, for 2c\_vs\_64zg, MARC enables the agents to learn the task much faster and at a higher level of performance. Both QMIX and VDN show much-improved performances when combined with MARC, with particular emphasis on VDN+MARC, where the improvement is outstanding. In 1o2r\_vs\_4r (as used in works like \cite{masia_2022,wang_ndq_2019}) (Fig. \ref{fig:env_d}) we observe a similar scene, where QMIX+MARC stays above the others, although all of them achieve average winning rates. In the case of MMM2, in Fig. \ref{fig:env_c} we can also see benefits of communication, although these are not as prominent as in the case of the previous scenarios. The effect of communication becomes evident mostly when we look at QMIX+MARC, where MARC has a strong positive impact. VDN+MARC also shows improvements, although it stays below QMIX+MARC. Interestingly, when we look at other communication methods we can see that, in general, TARMAC \cite{das_tarmac_2019} or MASIA \cite{masia_2022} do not perform well in some tasks. We also note the inconsistent performance of COMMNET, which performs reasonably in some cases, but fails in others (Fig. \ref{fig:env_c}). 

Taking a deeper look at the behaviours of the agents, the positive impact of MARC is further illustrated in Fig. \ref{fig:smac_behav} for 2c\_vs\_64z. We can see in this figure that agents that use MARC to communicate (bottom) tend to remain cooperating close to each other for the entire episode, while when they do not communicate (top), they eventually start moving away from each other. The behaviours when using QMIX+COMMNET (middle) support our previous observations, where agents tend to be together more often when they communicate, as with MARC. However, with QMIX+COMMNET, the agents assume a less optimal strategy where they spent long periods of time in the middle of the enemy team, losing health points and resulting in the elimination of one of the agents, leaving the other alone to finish the game. This suboptimal behaviour does not happen with MARC and this is reflected in the team performances in Fig. \ref{fig:main_res}.
\begin{figure*}[!t]
    \centering
    \includegraphics[width=\textwidth]{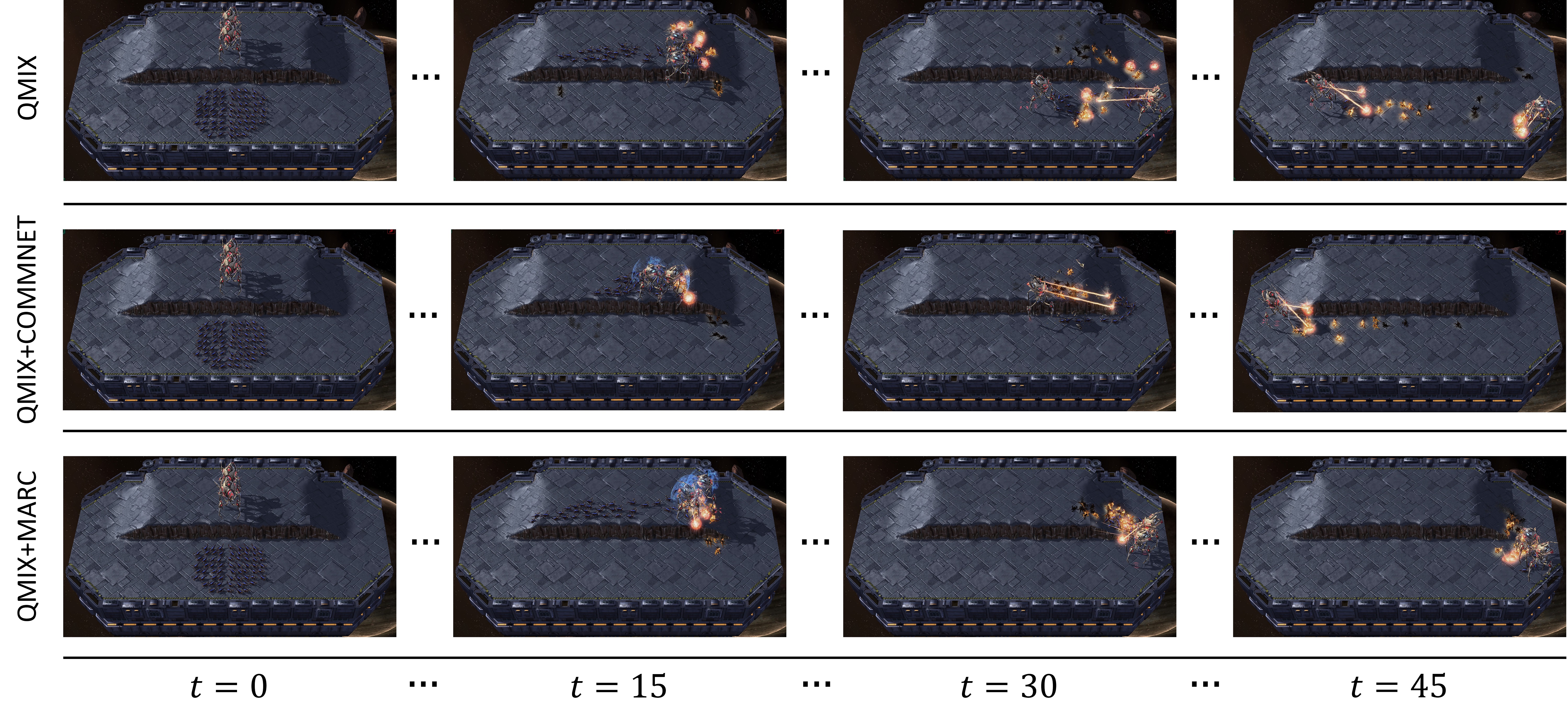}
    \caption{Learned behaviours for QMIX (top), QMIX+COMMNET (middle), and QMIX+MARC (bottom) after training in 2c\_vs\_64zg. When we use MARC, we can see that the agents will be always cooperating close to each other, while without communication they eventually move away from one another (top right). Additionally, we can see that with QMIX+COMMNET the agents communicate but learn a sub-optimal policy that results into one of them being eliminated (middle right).}
    \label{fig:smac_behav}
\end{figure*}

To further demonstrate the strength of the proposed method, we evaluate MARC in a PredatorPrey environment \cite{magym}. Previous works such as \cite{qtran_2019,liu_multi-agent_2021,deep_coord_2019} have demonstrated the importance of considering these scenarios that impose stronger punishments for non-cooperative behaviours. We consider a version of this environment where 4 agents must catch 2 moving prey in a $7\times7$ grid. At least two agents are needed to catch one prey, and when they do, the team receives a reward of $5\times N$, where $N$ is the number of agents. For each step there is a small penalty of $-0.01\times N$ and, most importantly, there is a team penalty of $p\times N$ that punishes the agents when one of them attempts to catch a prey alone. In Fig. \ref{fig:pp_res} we can see the rewards achieved by the experimented methods for different values of $p$. When we use MARC, the agents take advantage of the messages sent by the others and can solve the task. While for $p=-0.50$ communication does not seem to be critical, as we increase to $p=-0.75$ (Fig. \ref{fig:env_pp_b}) and $p=-1.0$ (Fig. \ref{fig:env_pp_c}), we can see that communication is necessary, and the agents cannot solve the task without it. Interestingly, while TARMAC and MASIA only managed to achieve positive rewards for $p=-0.50$ (Fig. \ref{fig:env_pp_a}), when we scale to $p=-0.75$ and $p=-1.0$, MARC and COMMNET are the only ones to be successful. Overall, we observe the inconsistency of methods like COMMNET and TARMAC, which do well in some cases, but completely fail in others. On the other hand, MARC shows consistently good performances in all the experimented scenarios. 

\subsection{Performance in Other Complex Scenarios}\label{app:more_perf_res}
To support the results presented in the previous subsection \ref{sec:main_exps}, we demonstrate now the strength of the proposed method in other relevant complex environments. We analyse our method in Lumberjacks and TrafficJunction (environments as illustrated in Fig. \ref{fig:app_envs_a} and \ref{fig:app_envs_b}). As discussed in the previous subsection, in some SMAC environments communication might not be as important as in other scenarios. Thus, in addition to the PredatorPrey results in Fig. \ref{fig:pp_res}, we carried additional experiments in other complex environments. However, we still stress the importance of these methods being able to solve a wide range of environments, and not only scenarios where communication is needed. In order to get a better understading of these additional environments, we present below a brief description for each one of them.

\textbf{Lumberjacks} (Fig. \ref{fig:app_envs_a}) consists of an environment where 4 agents must chop all the existing 12 trees in a $8\times 8$ map. In this environment, the agents receive a penalty of $-1$ every timestep and a reward of $+10$ when a tree is cut \cite{magym}. Each tree is assigned a random level between 1 and the number of agents, where this level represents the number of agents required at the same time to cut the tree. Importantly, the high step penalty and the elevated need for cooperation make this task challenging.

\textbf{TrafficJunction} (Fig. \ref{fig:app_envs_b}) represents a cross-shaped traffic junction where agents coming from 4 different entries in the junction must cross towards a pre-defined location at the end of another road after crossing the junction \cite{magym}. In our case, we have defined the number of agents to 10 agents, making this task very challenging. The agents received a penalty of $-0.01$ to incentivize them to keep moving and a penalty of $-10$ if they collide with another agent.


In the results depicted in Fig. \ref{fig:app_res_a} and Fig. \ref{fig:app_res_b} we observe that, in these complex environments, communication shows to have a strong impact and the proposed method demonstrates substantially improved performances over the baselines. This kind of environments has been used in other works to evaluate the strength of communication, such as in \cite{sukhbaatar_learning_2016}. Importantly, we note once again the robustness of MARC across different environments, while others might work well in some of them, but perform poorly in others (as discussed in the previous subsection \ref{sec:main_exps}).
\begin{figure}[!t]
    \centering
    \includegraphics[width=0.7\columnwidth]{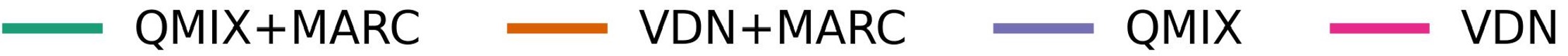}
    \\
    \vspace{0.00mm} 
    \subfigure[Lumberjacks]{\label{fig:app_res_a}\includegraphics[width=0.45\columnwidth]{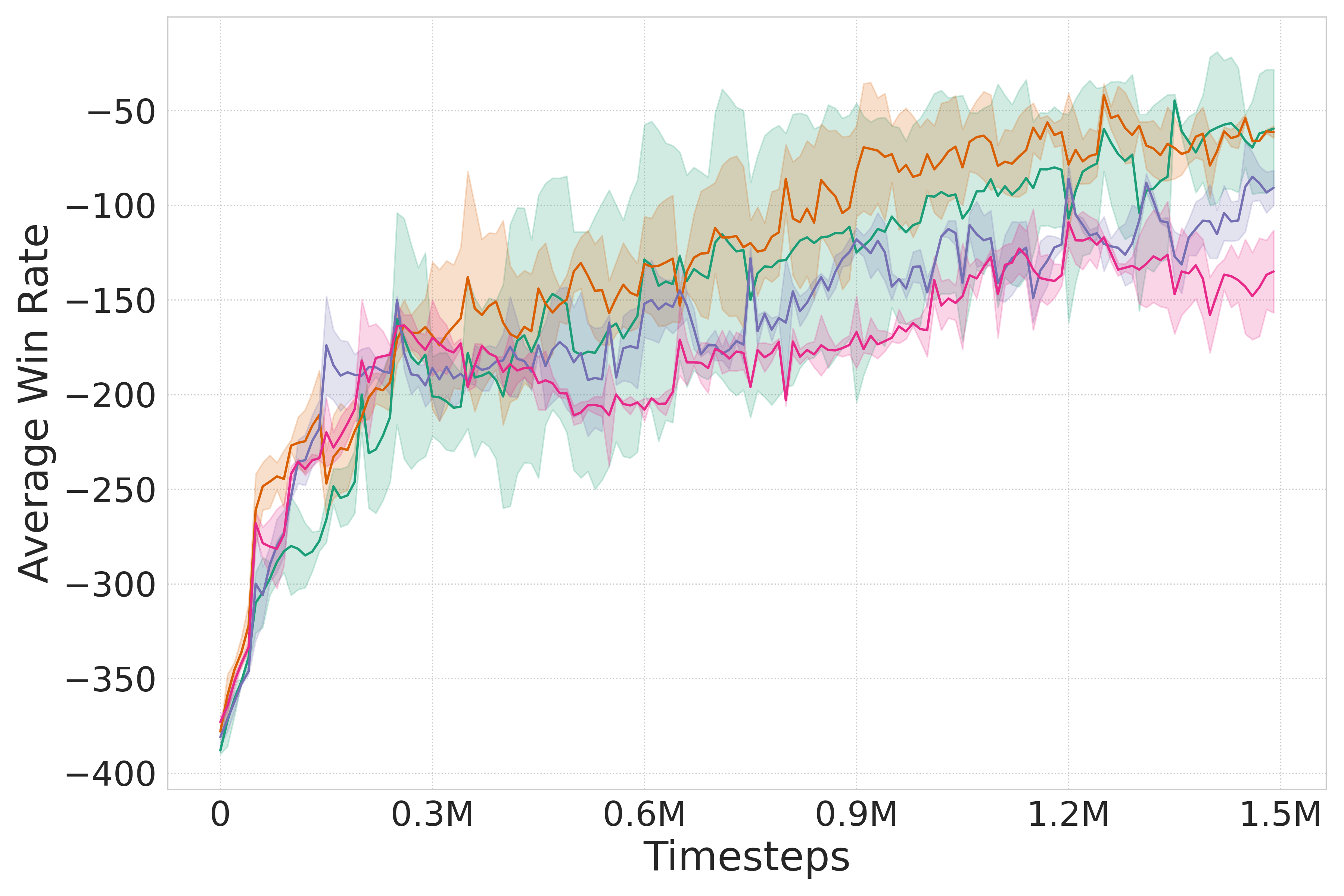}}
    \hfill
    \subfigure[Traffic Junction]{\label{fig:app_res_b}\includegraphics[width=0.45\columnwidth]{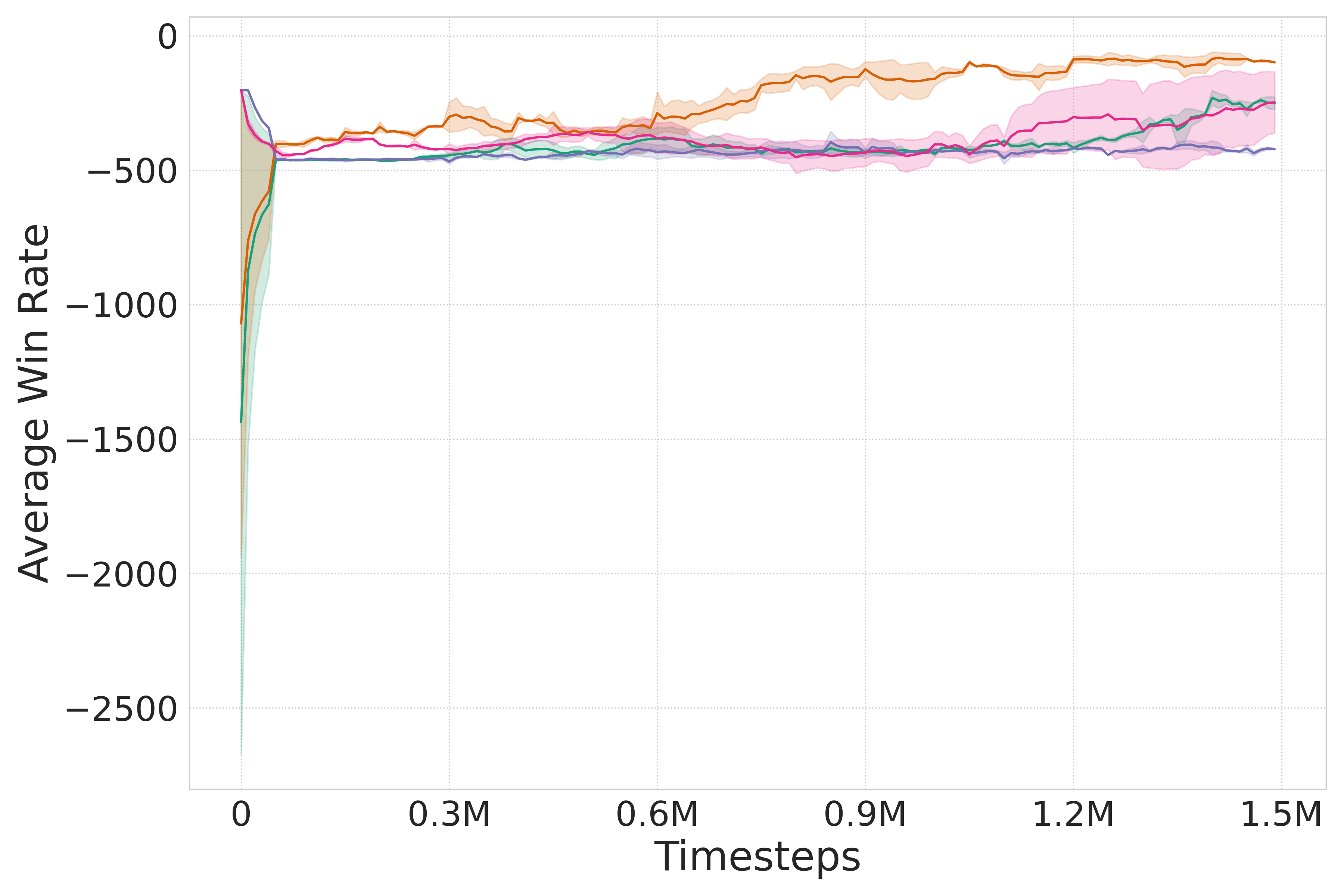}}
    \caption{Additional results in (a) Lumberjacks and (b) Traffic Junction. The plots (a) and (b) depict the average team reward over time.}
    \label{fig:app_res}
\end{figure}
\begin{figure}[!t]
    \centering
    \includegraphics[width=\columnwidth]{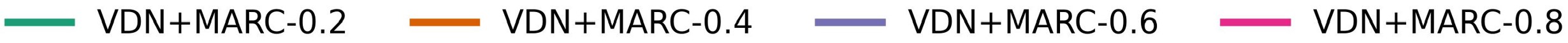}
    \\
    \vspace{0.00mm} 
    \subfigure[3s\_vs\_5z]{\label{fig:comp_vdn_env_a}\includegraphics[width=0.45\columnwidth]{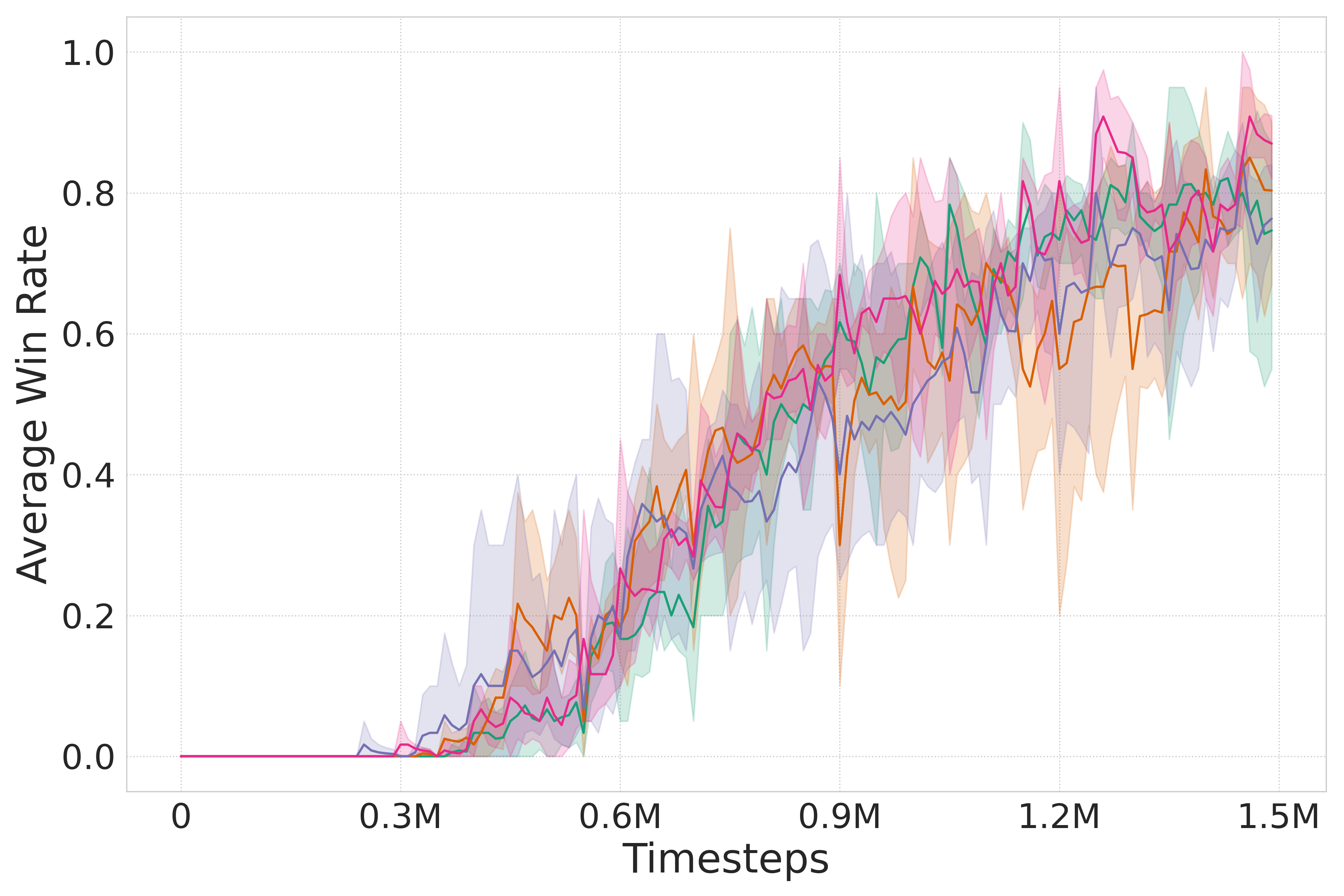}}
    \hfill
    \subfigure[2c\_vs\_64zg]{\label{fig:comp_vdn_env_b}\includegraphics[width=0.45\columnwidth]{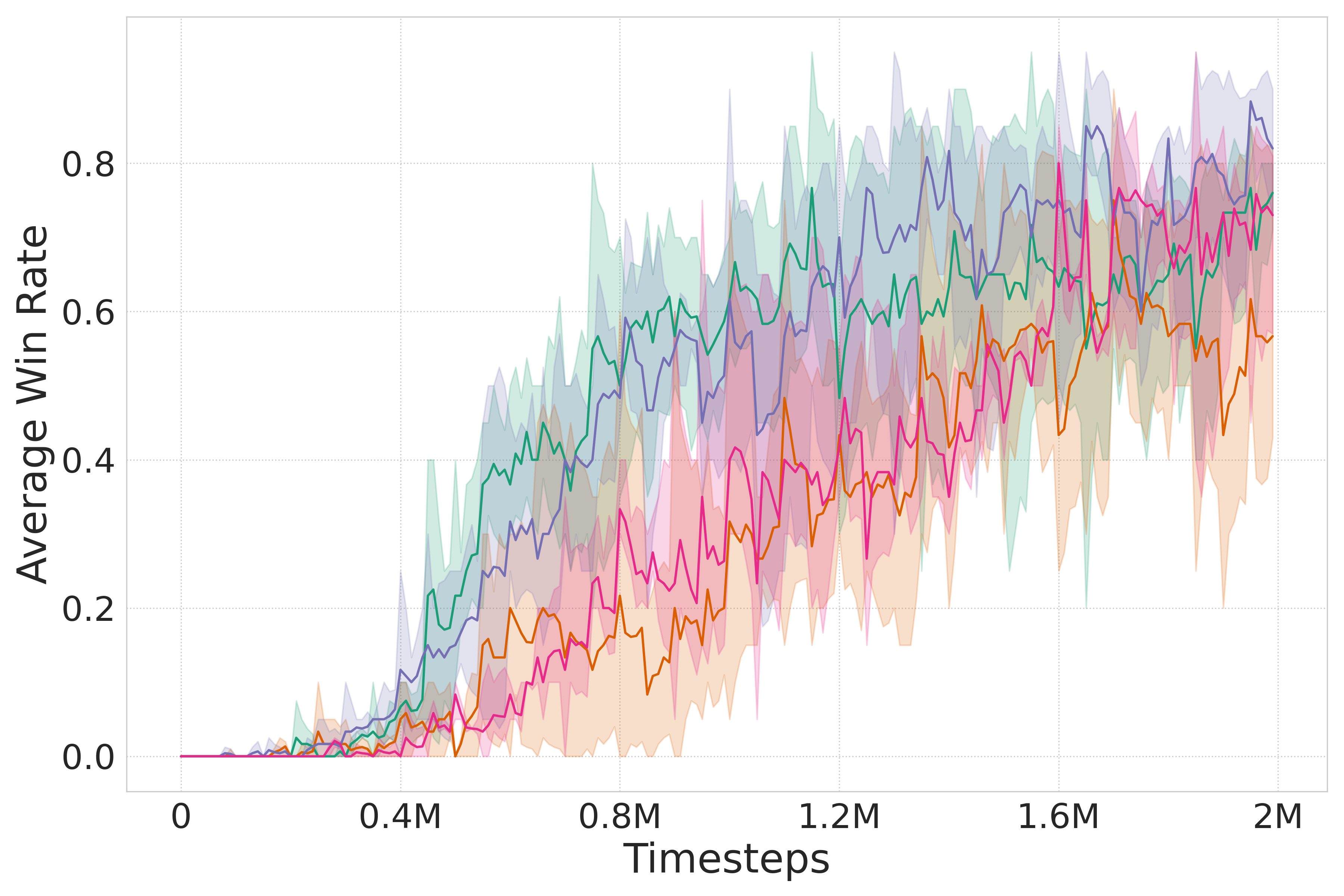}}
    \caption{Results for VDN+MARC when we use the DCT to compress the messages that the agents generate. To evaluate the impact of the lossy compression of the DCT, we compress the messages to 20\%, 40\%, 60\%, and 80\% of their original size, as illustrated in the figure.}
    \label{fig:vdn_comp_res}
\end{figure}

\subsection{Additional Results: Message Compression and Representation}\label{sec:exps_comp}
\subsubsection{Compressing Messages}
In the previous section, the results presented assume optimal conditions where communication among entities can be done in a lossless manner, i.e., where messages can be exchanged without loss of information. In this subsection, we present the results of communicating under a lossy communication channel, using the same proposed architecture. More specifically, we use the DCT to compress the messages into a smaller representation. Importantly, this does not affect the size of the communication endpoints, since the messages are reconstructed once they reach the destination, and thus there is no need to modify the size of the networks for different levels of compression. 
\begin{figure}[!t]
    \centering
    \includegraphics[width=\columnwidth]{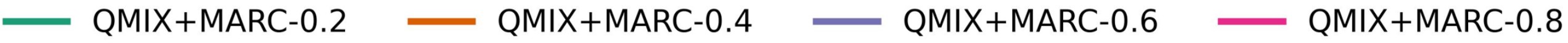}
    \\
    \vspace{0.00mm} 
    \subfigure[3s\_vs\_5z]{\label{fig:comp_env_a}\includegraphics[width=0.45\columnwidth]{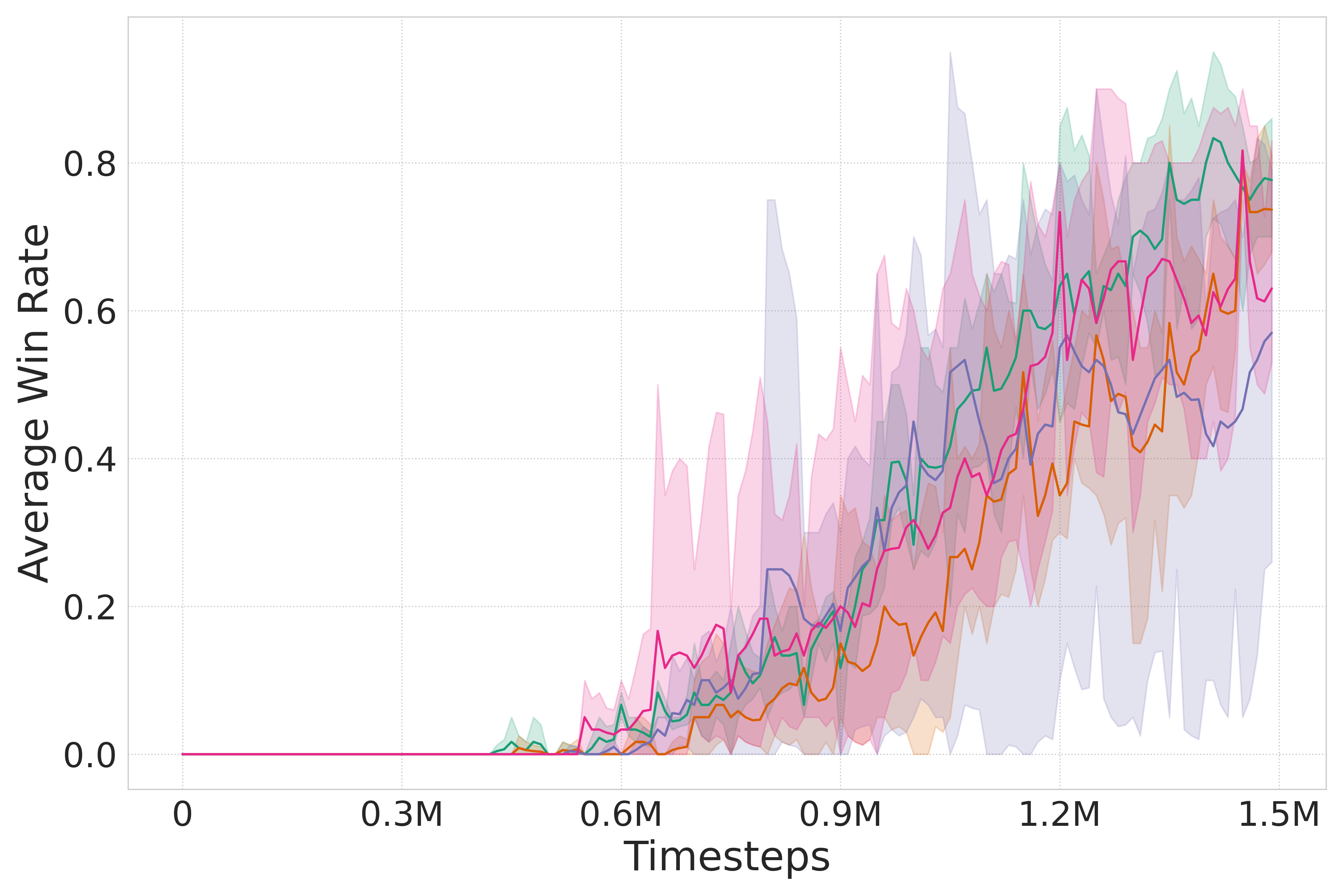}}
    \hfil
    \subfigure[2c\_vs\_64zg]{\label{fig:comp_env_b}\includegraphics[width=0.45\columnwidth]{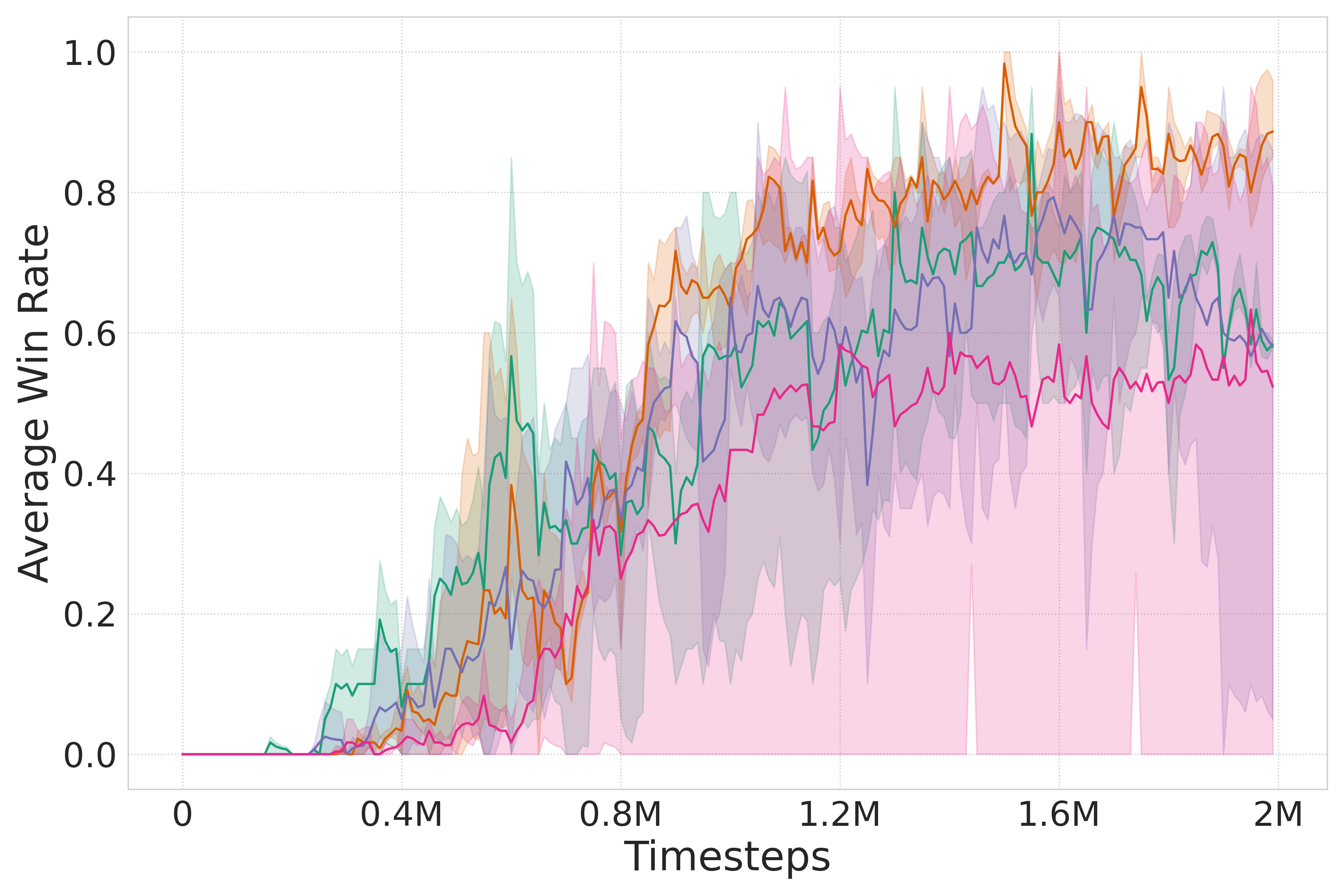}}
    \caption{Results for QMIX+MARC when we use the DCT to compress the messages that the agents generate. To evaluate the impact of the lossy compression of the DCT, we compress the messages to 20\%, 40\%, 60\%, and 80\% of their original size, as illustrated in the figure.}
    \label{fig:comp_res}
\end{figure}

Fig. \ref{fig:vdn_comp_res} illustrates the performances of VDN+MARC under different levels of DCT lossy message compression. In the figure, we can see that VDN+MARC can still learn even when using compressed messages, despite there is also a decrease of performance in both of the attempted scenarios. However, in Fig. \ref{fig:comp_vdn_env_b}, there is still a big increase of performance when compared to VDN without communication in this environment (as seen before in Fig. \ref{fig:main_res}). Importantly, this means that, even when using compressed and hence lightweight messages, the performance can still be improved with communication when it is necessary.

In addition to the effects of compression on VDN+MARC, we extend our experiments to evaluate the effect of compression also on QMIX+MARC. Fig. \ref{fig:comp_env_a} depicts the performances of QMIX+MARC under different levels of DCT lossy message compression. As expected, we can see that the performances will decrease as the level of message compression increases. However, it is important to note that the agents can still learn the tasks despite the high levels of compression. In some cases where communication channels might be constrained to a certain bandwidth, it is important to ensure that communication can still be leveraged to improve cooperation in MARL. 

To investigate how the compression method affects the messages learned by the agents, we have looked into the messages devised by the agents for a 40\% level of compression with QMIX+MARC in the 2c\_vs\_64zg environment. In Fig. \ref{fig:msgs_comp} we can see that the messages will occupy smaller ranges (mainly evident when closer to the end of the episode), than when compared to when there is no compression (which will be discussed in the next subsection). This suggests that the agents try to find more specific messages with less broad values when these must be compressed afterwards.

\begin{figure}[!t]
    \centering
    \includegraphics[width=\columnwidth]{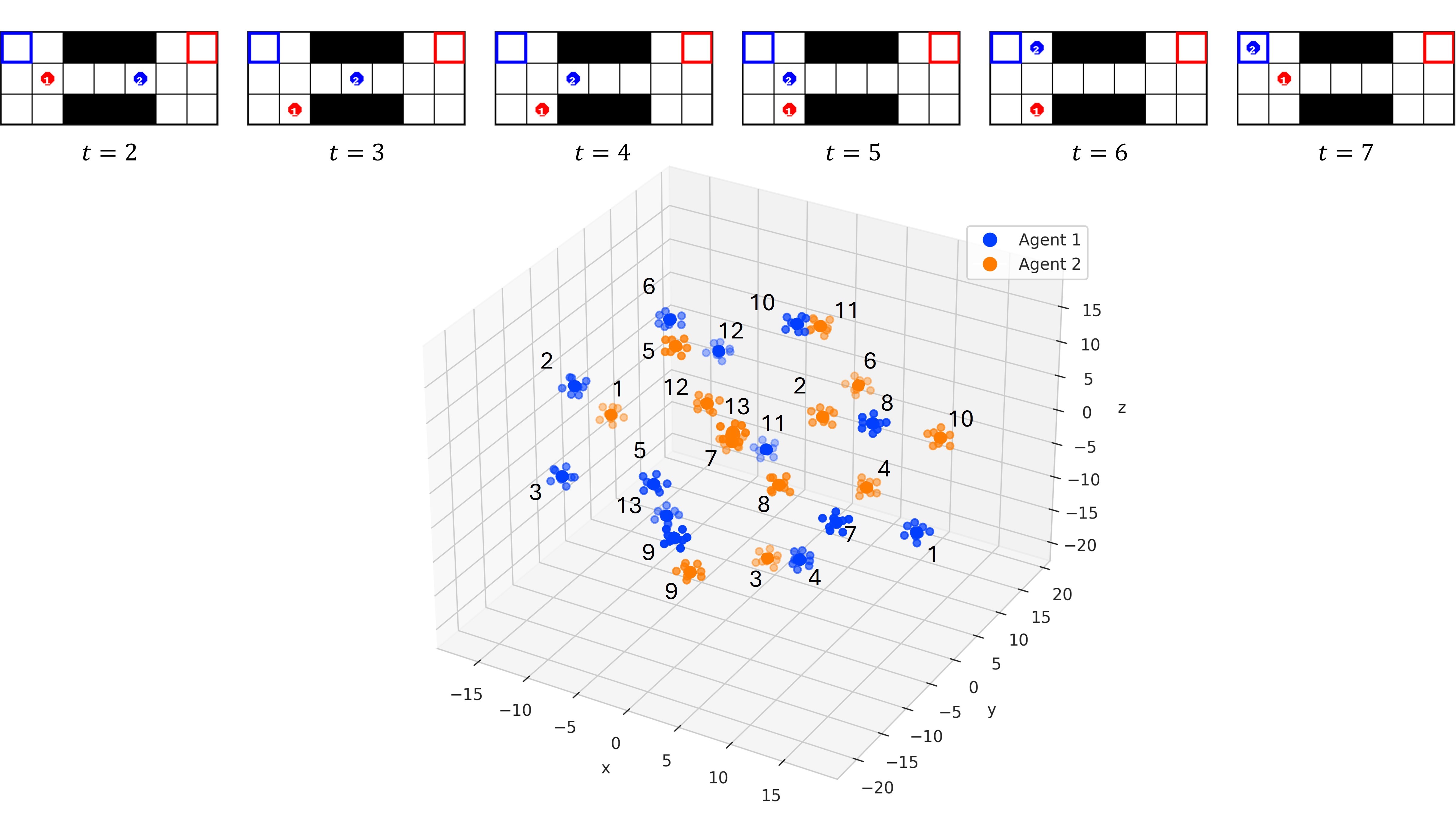}
    \caption{TSNE plot of the messages learned by each agent in a Switch task, after training using VDN+MARC. Each number in the plot represents the respective time step in the environment until termination, next to the message generated by each agent at that timestep. Note that, while the depicted episode lasted for 13 timesteps, above we show the states for the timesteps from 2 to 7 only.}
    \label{fig:switch_tsne}
\end{figure}

\subsubsection{Messages Learned}\label{add_res:msgs_learned}
One of the motivations for the MARC architecture is to enable the agents to produce more meaningful messages of their perceptions of the environment. To further support our motivations, we have trained VDN+MARC in the Switch environment from \cite{magym} where two agents need to reach their destination by crossing a corridor where only one of them fits at a time. Fig. \ref{fig:switch_tsne} shows the TSNE plot of the messages learned by the agents on a successful episode after training the team in this environment (note that the episode lasts for 13 timesteps but only the steps from 2 to 7 are illustrated). In the plot we can clearly see that, at each timestep, each agent learns a distinct message that contains a meaning that is attached to that particular observation. At the time of negotiating their passage ($t=[3, \dots, 5]$), in the screenshots above the plot we can see that one agent waits for the other to pass and only after will start to cross, while taking into account the messages received.
\begin{figure}[!t]
    \centering
    \subfigure[QMIX+MARC with 40\% DCT compression]{\label{fig:msgs_comp_a}\includegraphics[width=0.45\columnwidth]{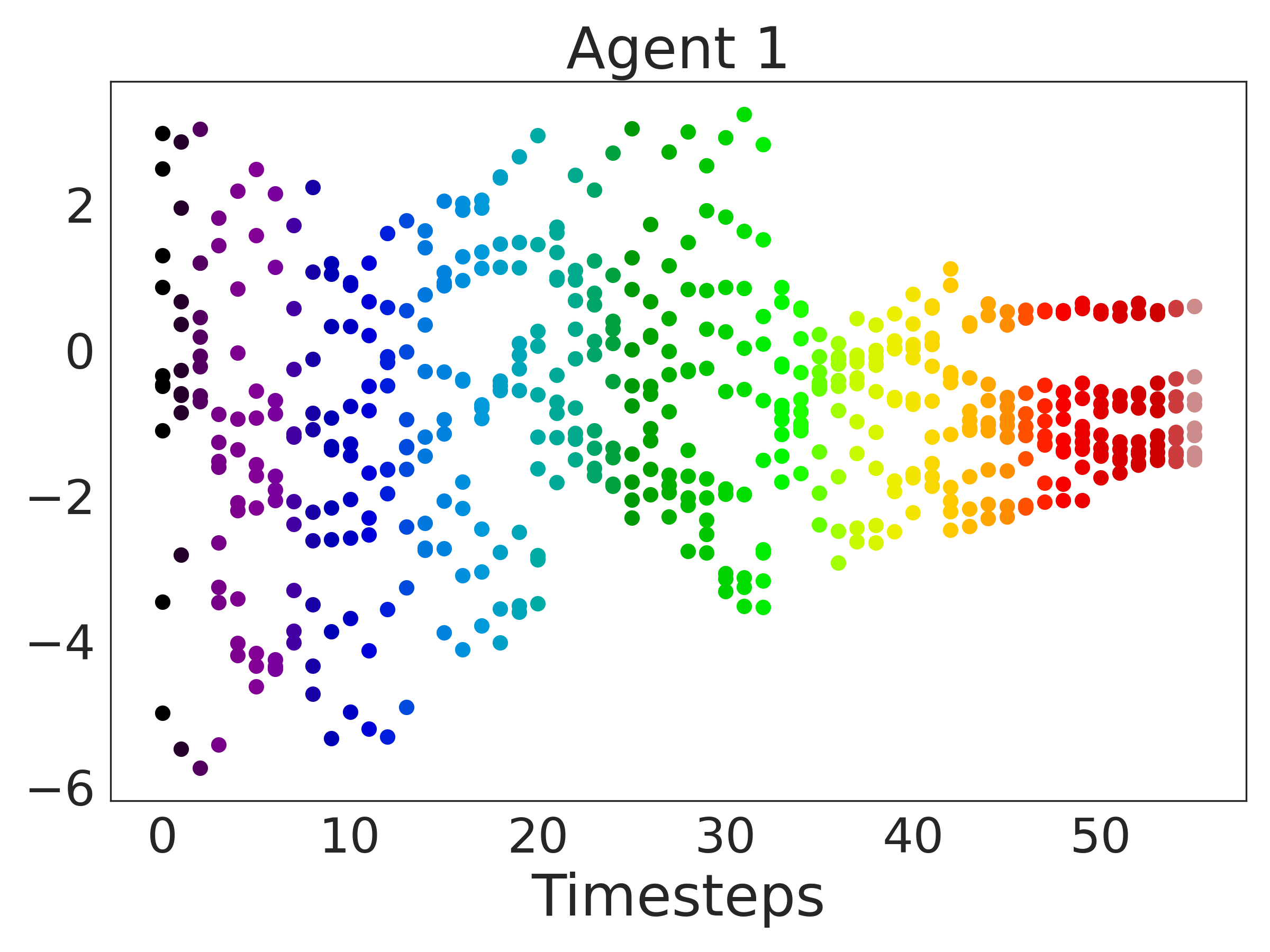}}
    \hfill
    \subfigure[QMIX+MARC with 40\% DCT compression]{\label{fig:msgs_comp_b}\includegraphics[width=0.45\columnwidth]{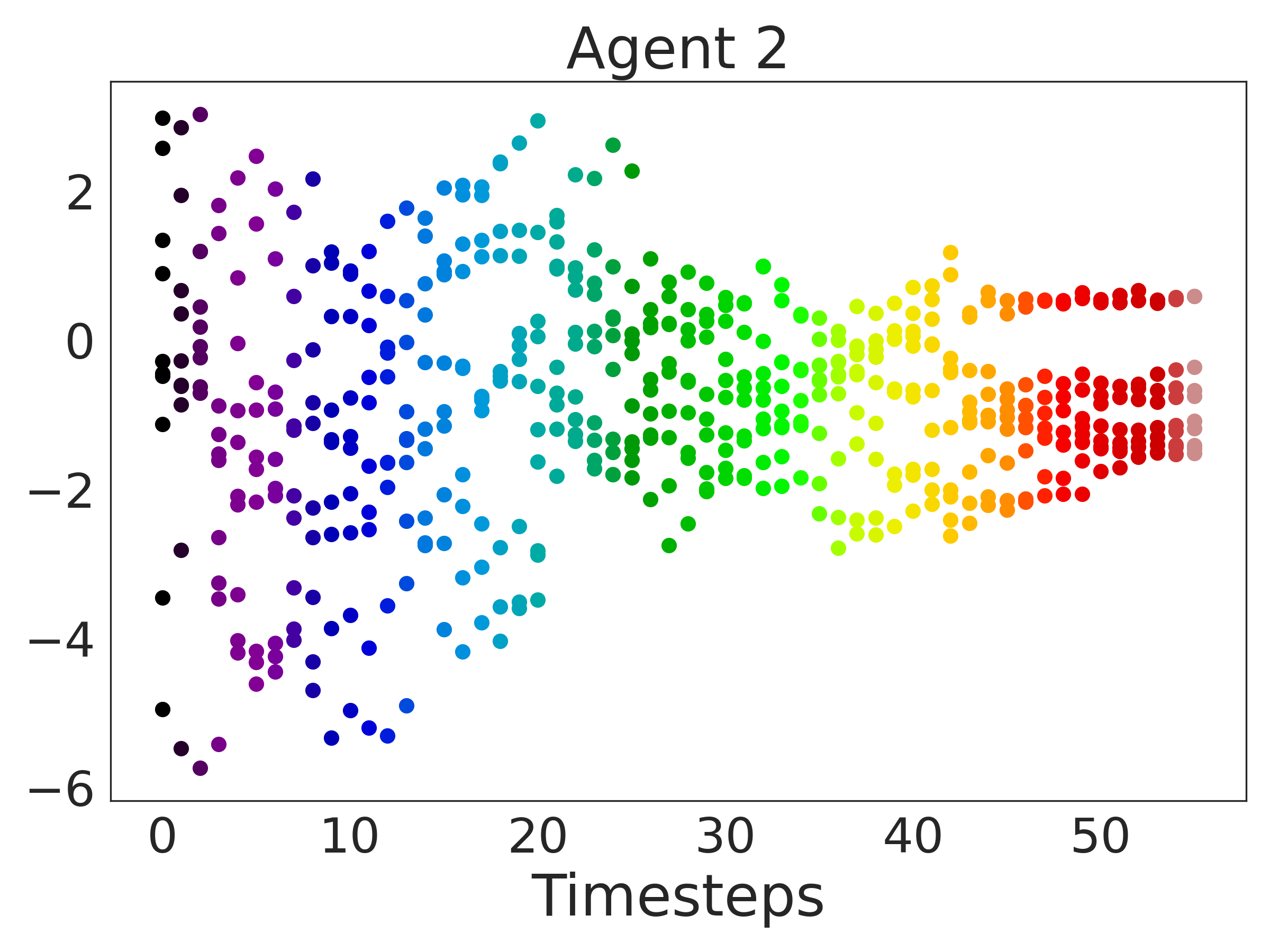}}
    \caption{Messages generated by the two agents over the course of a successful 2c\_vs\_64zg episode, after being trained with QMIX+MARC under a message compression level of 40\% with the DCT (messages plotted after being decompressed at the destination using the IDCT).}
    \label{fig:msgs_comp}
\end{figure}
    
\subsection{Ablations: Message Regularizer}\label{sec:exps_msg_reg}
We have previously described how MARC uses a message regularizer to improve the quality of the messages learned during training. In this subsection, we analyse the impact of the message regularizer on the messages produced in the learning problem.  
    
To analyse the messages learned, we have saved the trained networks of QMIX+MARC for the 2c\_vs\_64zg task, both using and not using the regularizer (purple box in Fig.\ref{fig:net_arch}). Fig. \ref{fig:msgs_abl} shows the messages produced by the two agents involved in 2c\_vs\_64zg during a successful episode, i.e., an episode where the agents win the game. We can see in Fig. \ref{fig:msgs_abl_c} and \ref{fig:msgs_abl_d} that the messages produced when the agents use the message regularizer are compact within a smaller interval (around $\left[-4,4\right]$) when compared to the messages produced when we do not use the regularizer (Fig. \ref{fig:msgs_abl_a} and \ref{fig:msgs_abl_b}). In the latter case, the messages produced assume values that lie inside a much larger range (around $\left[-10,10\right]$), meaning that the possible kind of messages learned by the agents is more ambiguous, and making the task more difficult for them since the messages will not be as precise and meaningful as in the case of the message regularizer. The values of the differential entropy $H$ for the messages produced that are shown in the caption of the figures also support our observations that messages that were produced with the help of the regularizer will contain less uncertainty and ambiguity. These observations support the results presented in subsubsection \ref{add_res:msgs_learned} showcasing messages other messages produced through MARC.
    \begin{figure}[!t]
        \centering
        \subfigure[\textbf{No} regularizer.\newline $H=2.065$.]{\label{fig:msgs_abl_a}\includegraphics[width=0.23\textwidth]{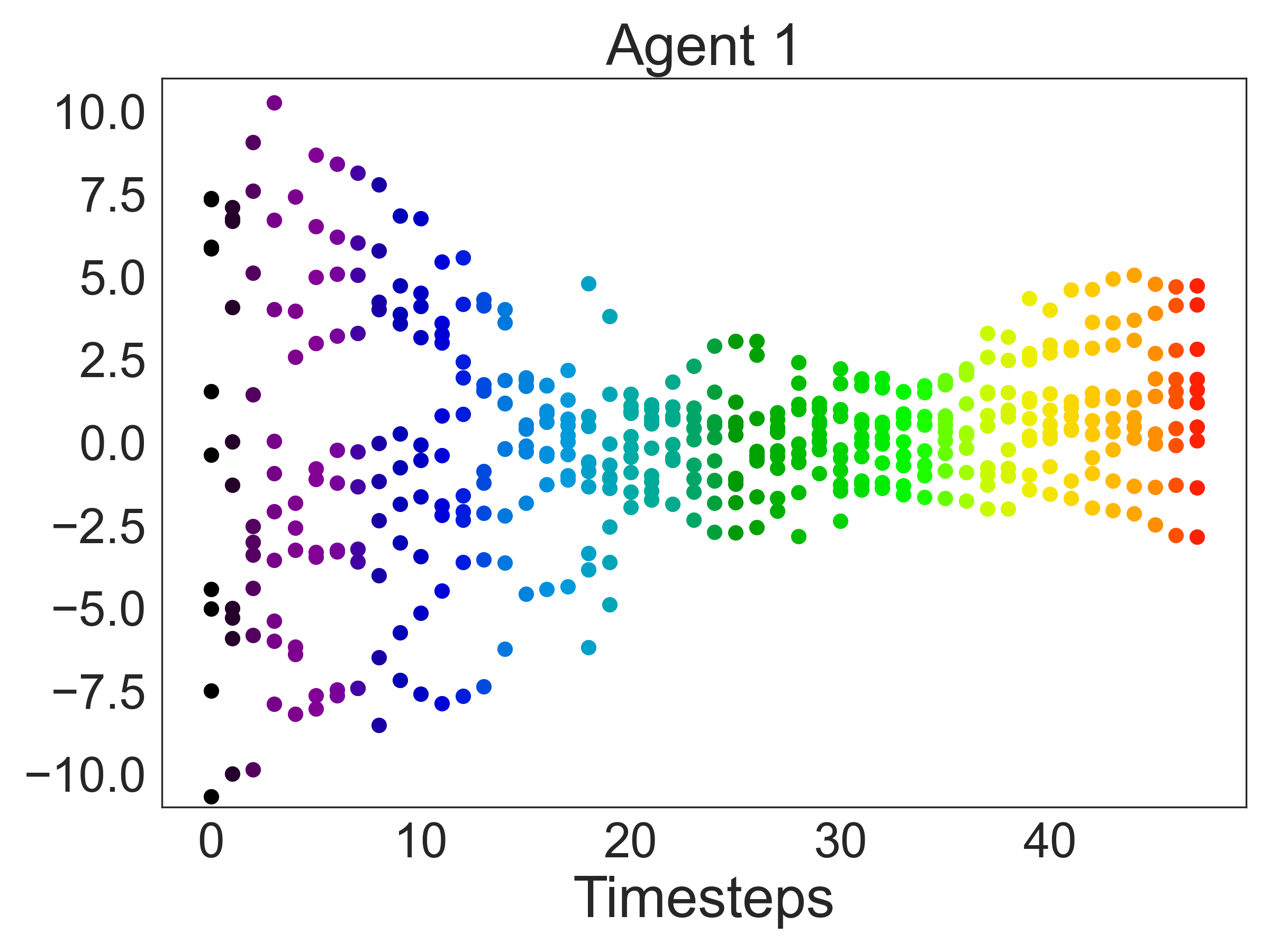}}
        \hfill
        \subfigure[\textbf{No} regularizer.\newline $H=2.172$.]{\label{fig:msgs_abl_b}\includegraphics[width=0.23\textwidth]{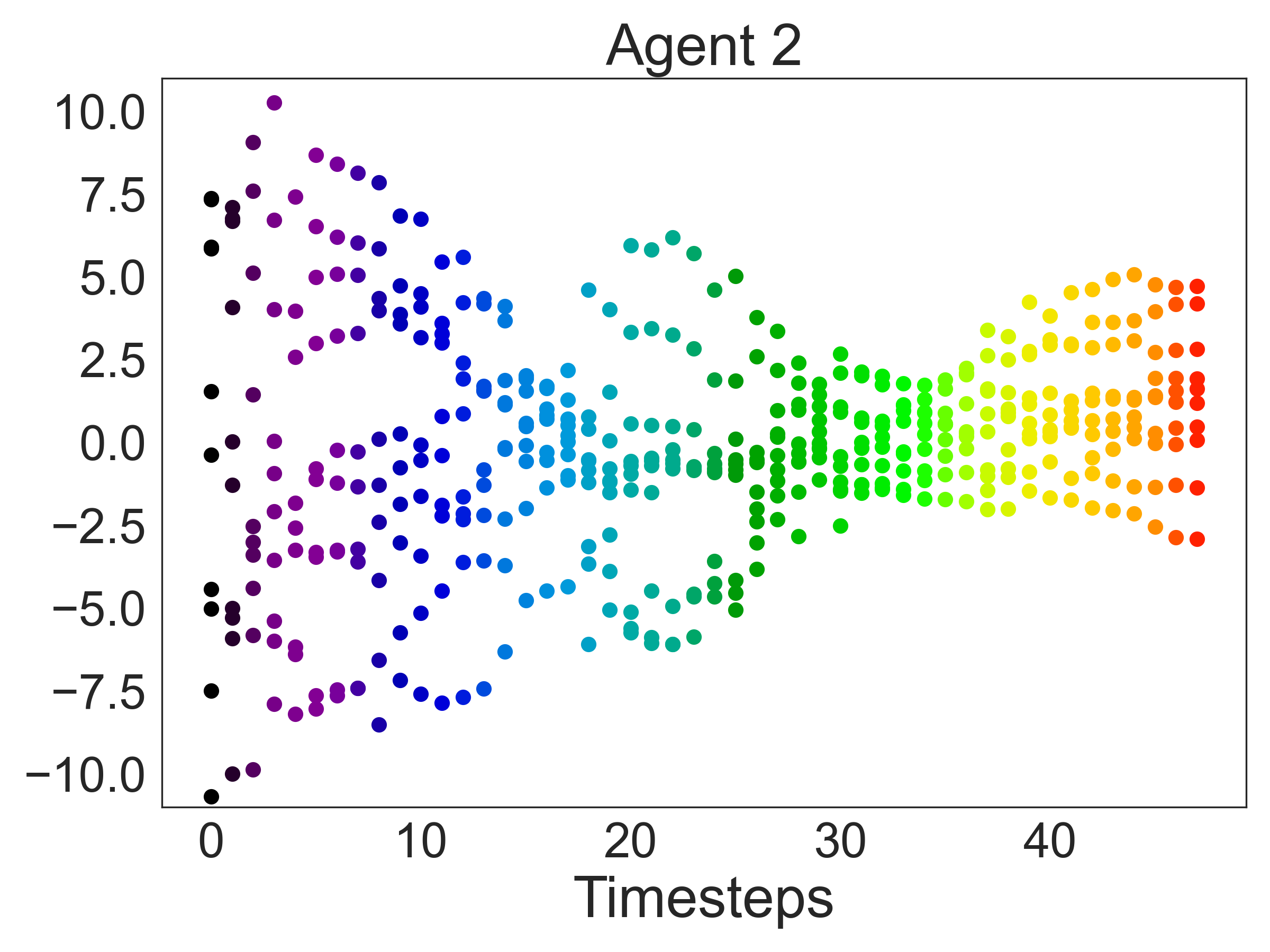}}
        \hfill
        \subfigure[\textbf{With} regularizer.\newline $H=1.279$.]{\label{fig:msgs_abl_c}\includegraphics[width=0.23\textwidth]{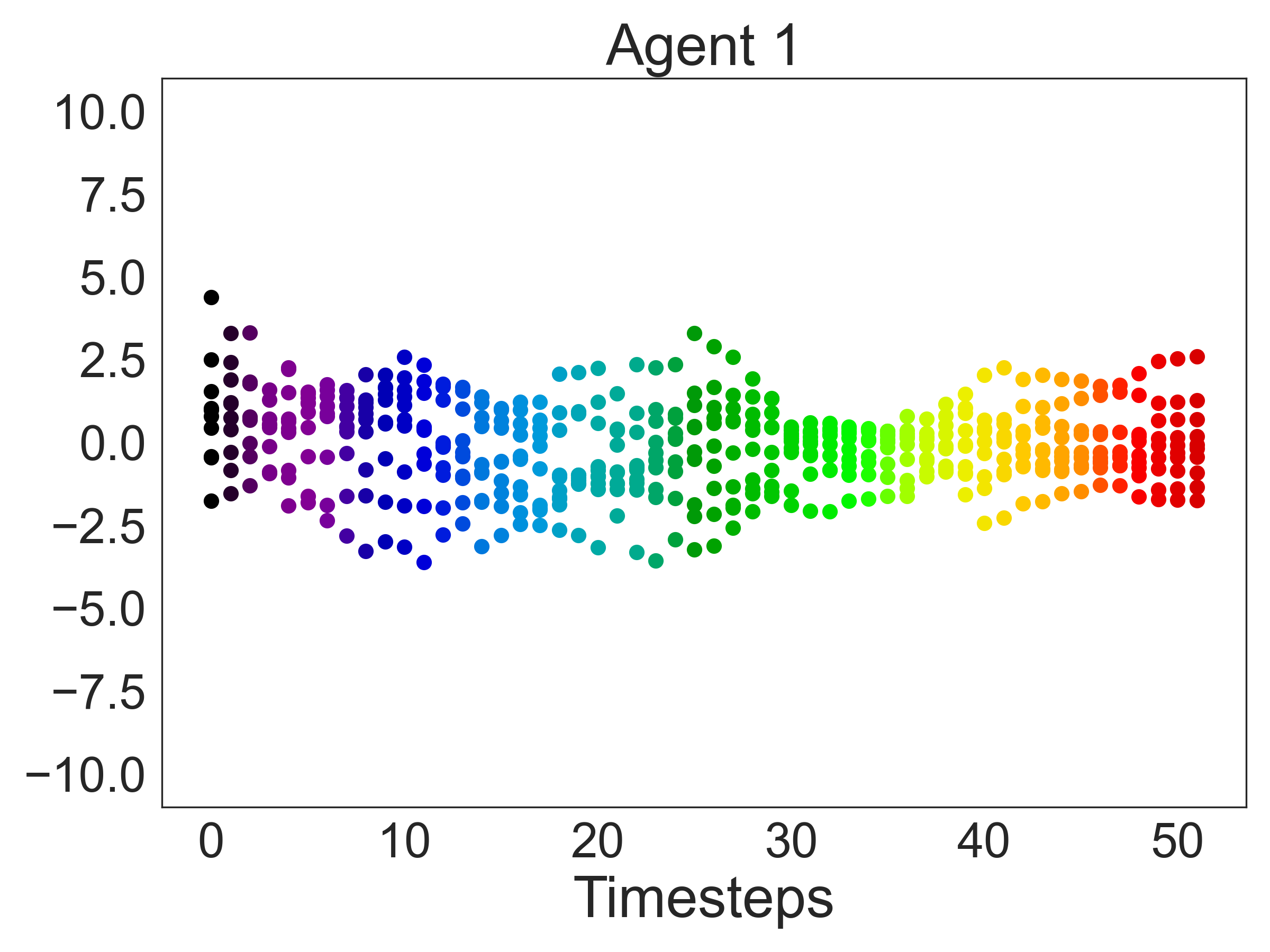}}
        \hfill
        \subfigure[\textbf{With} regularizer.\newline $H=1.293$.]{\label{fig:msgs_abl_d}\includegraphics[width=0.23\textwidth]{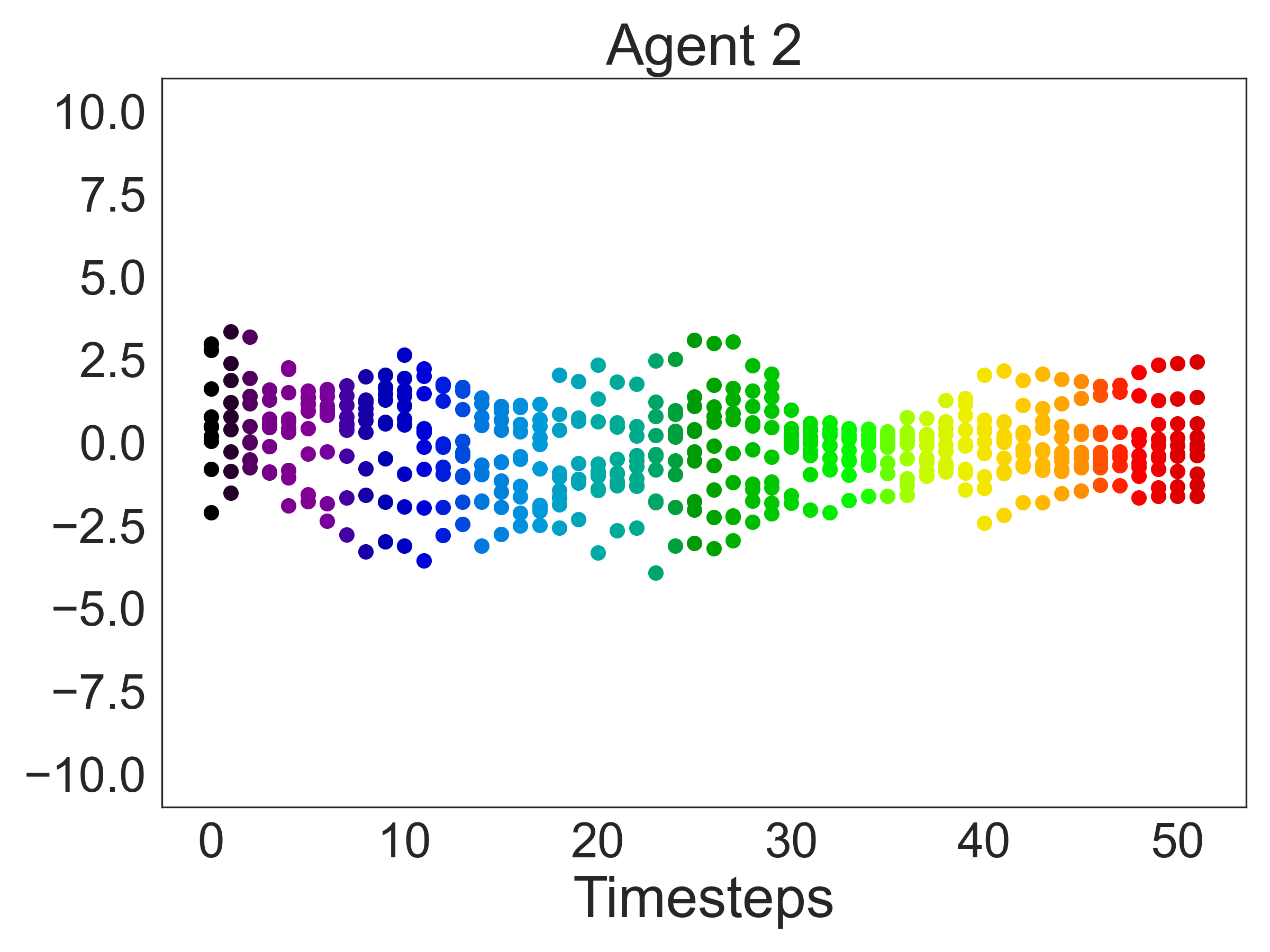}}
        \caption{Messages learned by the two agents in 2c\_vs\_64zg throughout one successful episode with QMIX+MARC when the message regularizer is removed (a-b), and when it is used (c-d). The values of $H$ denote the differential entropy of the messages.}
        \label{fig:msgs_abl}
    \end{figure}

\section{Conclusion and Future Work}\label{sec:conc}
Communication frameworks in Multi-Agent Reinforcement Learning (MARL) represent a transformative approach to developing cooperative strategies within decentralized systems. By enabling agents to synthesize and share local observations, communication facilitates a higher level of behavioral synchronization and predictive coordination. In this work, we introduced \textbf{Multi-Agent Regularized Communication (MARC)}, a modular framework compatible with diverse value-function factorization methods. Our experimental results demonstrate that MARC significantly enhances performance in complex tasks by fostering the emergence of highly representative communication protocols through its novel information-inspired regularization mechanism.

Furthermore, we investigated a critical gap between MARL theory and the physical constraints of real-world deployment: limited communication bandwidth. By integrating the Discrete Cosine Transform (DCT) for message compression, we demonstrated that robust distributed intelligence can be maintained even under significant information loss. This finding validates that communication-based cooperation is viable for resource-constrained hardware and high-latency environments where spectral efficiency is paramount.

While MARC provides a robust foundation, several avenues for future research remain. We aim to investigate the impact of communication across varying network capacities and heterogeneous agent architectures. Additionally, while the current work operates within the CTDE paradigm, future iterations will explore fully decentralized training schemes to further relax the reliance on centralized mixers. Finally, we intend to study how adaptive message compression affects learning across different network scales and extend these findings to practical robotic testbeds to showcase the effects of communicating in dynamic, real-world environments.

\bibliography{biblio}
\bibliographystyle{IEEEtran}

\vfill

\end{document}